\documentclass{article}

\usepackage{microtype}
\usepackage{graphicx}
\usepackage{subfigure}
\usepackage{booktabs} 
\usepackage{makecell}
\usepackage{multirow}

\usepackage{amsmath}
\usepackage{amsfonts}
\usepackage{amsthm}
\newtheorem{thm}{Theorem}

\newtheorem{cor}{Corollary}
\newtheorem{ass}{Assumption}
\newtheorem{definition}{Definition}

\newcommand{\ts}[1]{^{(#1)}}
\newcommand{\norm}[1]{\left\|{#1}\right\|}

\usepackage{hyperref}

\graphicspath{{./figures/}}

\usepackage[accepted]{icml2020}

\icmltitlerunning{Do RNN and LSTM have Long Memory?}

\begin{document}

\twocolumn[
\icmltitle{Do RNN and LSTM have Long Memory?}



\icmlsetsymbol{equal}{*}

\begin{icmlauthorlist}
\icmlauthor{Jingyu Zhao}{hku}
\icmlauthor{Feiqing Huang}{hku}
\icmlauthor{Jia Lv}{hw}
\icmlauthor{Yanjie Duan}{hw}
\icmlauthor{Zhen Qin}{hw}
\icmlauthor{Guodong Li}{hku}
\icmlauthor{Guangjian Tian}{hw}
\end{icmlauthorlist}

\icmlaffiliation{hku}{Department of Statistics and Actuarial Science, The University of Hong Kong, Hong Kong, China}
\icmlaffiliation{hw}{Huawei Noah's Ark Lab, Hong Kong, China}

\icmlcorrespondingauthor{Guodong Li}{gdli@hku.hk}

\icmlkeywords{Recurrent Networks, Geometric Ergodicity, Long Memory, Fractional Integration}

\vskip 0.3in
]



\printAffiliationsAndNotice{}  

\begin{abstract}
The LSTM network was proposed to overcome the difficulty in learning long-term dependence, and has made significant advancements in applications. 
With its success and drawbacks in mind, this paper raises the question - do RNN and LSTM have long memory? 
We answer it partially by proving that RNN and LSTM do not have long memory from a statistical perspective. 
A new definition for long memory networks is further introduced, and it requires the model weights to decay at a polynomial rate. 
To verify our theory, we convert RNN and LSTM into long memory networks by making a minimal modification, and their superiority is illustrated in modeling long-term dependence of various datasets.
\end{abstract}

\section{Introduction}

Sequential data modeling is an important topic in machine learning, and it has been well addressed by a variety of recurrent networks. 
In the meanwhile, modeling long-range dependence remains a key challenge. 
Bengio et al. \yrcite{bengio1994learning} concluded that learning long-term dependencies by a system $ y_t = M (y_{t-1}) + \varepsilon_t $ with gradient descent is difficult. 
Later, much effort was spent on proposing new networks to fight against vanishing gradients, such as LSTM \cite{hochreiter1997long} and GRU networks \cite{cho2014learning}. 

LSTMs have shown their success on both synthetic classification tasks, such as the ``parity" problem and the ``latching" problem \cite{hochreiter1997long}, and countless applications, including speech recognition and machine translation. 
Despite its success, questions concerning the ability of LSTM in handling long-term dependence arise. 
For example, Cheng et al. \yrcite{cheng2016long} pointed out that the update equation of LSTM is Markovian, which is consistent with the system assumed in Bengio et al. \yrcite{bengio1994learning}. 
Greaves-Tunnell \& Harchaoui \yrcite{greaves2019statistical} utilized statistical tests and concluded that LSTM cannot fully represent the long memory effect in the input, nor can it generate long memory sequences from white noise inputs. 

Does LSTM have long memory? It is difficult to judge if long memory is assessed heuristically by the performance of LSTM on certain tasks. 
However, long memory happens to be a well-defined and long-existing terminology in statistics. 
From a statistical perspective, our investigation yields an affirmative reply. 
Our first contribution is to prove that a recurrent network process with Markovian update dynamics does not have long memory under mild conditions.

Nevertheless, the statistical definition of long memory is neither applicable nor easily transferable to general neural networks. 
The main obstacle is the existence of non $i.i.d.$ exogenous input to the network, violating a key assumption in the study of stochastic processes. 
The long or short memory effect in the input confounds the memory property of the output sequence. 
Thus, we propose a new definition for the long memory network process, which aims to characterize the ability of a network process to extract long-term dependence from the input.

We further explore the possible implications of our theory in practice. 
Our analysis implies that we cannot assume that information can be stably stored in the hidden states of a recurrent network with Markovian updates. 
Subsequently, providing the hidden states with more past information is a natural strategy to avoid the vanishing gradient problem.


In terms of network design, we propose a new memory filter component, which is controlled by a learnable memory parameter $d$.
The filter uses $d$ to deduce dynamic weightings on historical observations of arbitrary lengths and feeds past information to the hidden units. 
We incorporate this memory filter structure into RNN and LSTM and introduce Memory-augmented RNN (MRNN) and Memory-augmented LSTM (MLSTM). 
By adding this memory filter, the modified models achieve the ability to approximate the fractional differencing effect, 
which is a type of long memory effect, and retain the flexibility and expressiveness of neural networks.  

The main contributions of this work are below:
\begin{itemize}
	\item We provide theoretical conditions for recurrent networks, such as RNN and LSTM, to have short memory. 
	
	\item We propose a new definition for long memory network processes under a statistical motivation, and make it applicable to neural networks. 
	
	\item We propose a long memory filter component for network design. Modifying RNN and LSTM with this filter allows the networks to better model sequences with long memory effects.
	
	\item We perform numerical studies on several datasets to illustrate the advantages of our proposed models.
\end{itemize}

\subsection{Related Work}
Besides the aforementioned papers, the following works are also related to our paper in different aspects.

\textbf{Long memory phenomenon and statistical models}
Long memory effect can be observed in many datasets, including language and music \cite{greaves2019statistical}, financial data \cite{ding1993long, bouri2019modelling}, dendrochronology and hydrology \cite{beran2016long}. 
There are many statistical results and applications on modeling datasets with long memory \cite{grange1980introduction, hosking1981fractional, ding1993long, torre2007detection, bouri2019modelling, musunuru2019modeling, sabzikar2019parameter}.
Admittedly, these models enjoy rigorous statistical properties. However, they appear to be quite rigid and inflexible for real-world applications in comparison to neural networks. 

\textbf{Analysis of LSTM networks}
There is limited theoretical analysis of LSTM networks in the literature.
Miller \& Hardt \yrcite{miller2018stable} studied RNN and LSTM as dynamic systems and proved that $r$-step LSTM is a stable procedure, which implies that LSTM cannot model long-range dependence well. 
There also exists some work that tried to analyze LSTM numerically.
For example, Levy et al. \yrcite{levy2018long} studied which component of LSTM contributes the most to its success by ablation experiments, and concluded that the memory cell is largely responsible for the performance of LSTM. 
This is consistent with our view that the core of LSTM is the update equation of the cell state, which takes the form of a varying coefficient vector AR(1) model. 

\textbf{Recurrent networks for long-term dependencies}
Examples include the Hierarchical RNN \cite{el1996hierarchical}, where several layers of state variables are constructed at different time scales, and direct linkages are added between distant historical and current observations. 
Zhang et al. \yrcite{zhang2016architectural} propose the skip connections among multiple timestamps to account for long-range temporal dependencies. 
The idea is further extended by Chang et al. \yrcite{chang2017dilated} with the Dilated RNNs, where the number of parameters is reduced to enhance computation efficiency.
The NARX RNN \cite{lin1996learning} revised the hidden states of an RNN to follow a Nonlinear AutoRegressive model with eXogenous variables (NARX). 
Based on NARX RNN, MIxed hiSTory RNNs (MIST RNNs) \cite{dipietro2017analyzing} was proposed to reduce the number of trainable weights and link to hidden units in the distant past.
The recurrent weighted average (RWA) model \cite{ostmeyer2019machine} is also considered very akin to our proposed models.
Thus, we compare our models with MIST RNN and RWA in the experiments. 

\textbf{Memory networks and attention} 
Another popular trend of modeling long-range dependence lies in creating an external memory unit. 
The main controller interacts with the stand-alone memory unit through reading and writing heads. 
Neural Turing Machine \cite{graves2014neural} is perhaps the most famous example of such construction. 
Based on this idea, the attention mechanism \cite{vaswani2017attention} is also proposed. 
These models have gone beyond a recurrent nature and thus are not compared with our proposed models.

\textbf{Finite impulse response (FIR) filters and signal processing}
Our proposed memory filter can be viewed as a special FIR, and we propose to add it to RNN at the input and LSTM at the cell
states. 
Literatures on using filters on inputs include TFLSTM \cite{li2017spectro}, SRU \cite{oliva2017statistical} and Convolutional RNN \cite{keren2016convolutional}, while those for filters on hidden states can be found in NARX RNN \cite{lin1996learning} and Low-pass RNN \cite{stepleton2018low}.
Our filter is directly motivated by the ARFIMA model in statistics and shares the same root with the fractional-order filters \cite{radwan2008first, radwan2009generalization, radwan2009stability} in signal processing. 
The fractional-order filters enjoy similar properties with ours in terms of modeling long-range dependencies.

\section{Memory Property of Recurrent Networks}

\subsection{Background}

For a stationary univariate time series, there exists a clear definition of long memory (or long-range dependency) in statistics \citep{beran2016long}, and we state it below.
It is noticeable that Greaves-Tunnell \& Harchaoui \yrcite{greaves2019statistical} utilized the following definition to claim the long memory in language and music data.

\begin{definition} \label{def1}
	Let $\{X_t, t\in \mathbb{Z}\} $ be a second-order stationary univariate process with autocovariance function $ \gamma_X(k)$ for all $k \in \mathbb{Z} $ and spectral density $ f_X(\lambda) = (2\pi)^{-1}\sum_{k=-\infty}^{\infty} \gamma_X(k) \exp{(-ik\lambda)}$ for $\lambda \in [-\pi, \pi]$.
	Then $ \{X_t\} $ has (a) long memory, or (b) short memory if 
	\begin{equation} \label{eq:mem-acf}
	(a) \sum_{k=-\infty}^{\infty} \gamma_X(k) = \infty, ~ or ~ (b) ~ 0 < \sum_{k=-\infty}^{\infty} \gamma_X(k) < \infty.
	\end{equation}
	In terms of spectral densities, the conditions are, as $ |\lambda| \rightarrow 0 $, 
	\begin{equation} \label{eq:mem-specden}
	(a) f_X(\lambda) \rightarrow \infty, ~ or ~ (b) ~  f_X(\lambda) \rightarrow c_f \in (0, \infty).
	\end{equation}
\end{definition}

There are many types of long memory processes, and the fractionally integrated process is the most commonly used one in real applications; see \citet{beran2016long}.
Let $ B $ be the backshift operator, defined as $ B^j X_t = X_{t-j}$ for a $j\geq 0 $, applicable to all random variables in a time series $ \{X_t\} $. 
The fractionally integrated process \citep{hosking1981fractional} is defined as
$$ (1-B)^d Y_t = X_t, $$
where
\begin{equation} \label{eq:1-Bd}
(1-B)^d = \sum_{j = 0}^{\infty} \frac{\Gamma(d+j)}{j!\Gamma(d)}B^j =: \sum_{j = 0}^{\infty} w_j(d) B^j,
\end{equation}
and $ \Gamma(\cdot) $ is the gamma function.
The sequence $\{X_t\}$ is usually assumed to follow an ARMA model to capture the short-range dependency, and it then leads to a fractional ARIMA process of $\{Y_t\}$ \cite{grange1980introduction}.
It can be verified that
\[
w_j(d) \sim j^{-d-1}
\]
as $j\rightarrow\infty$, and $\gamma_X(k)\sim |k|^{2d-1}$ as $k\rightarrow\infty$.
As a result, time series $\{Y_t\}$ exhibit long memory when $ d \in (0, 0.5)$, and larger $d$ indicates longer memory. 
In other words, the long-range dependence of fractionally integrated processes can be characterized by the parameter $d$, which hence is called the memory parameter.

For multivariate time series, its definition of long memory is not unique, and a usual way is to check the long-range dependency of each component \citep{greaves2019statistical}.
Accordingly the multivariate fractionally integrated process $\{Y_t\}$ can be defined as $(I-\mathcal{B})^{d} Y_t=X_t$, where $Y_t\in \mathbb{R}^q$, $X_t\in \mathbb{R}^q$, $d=(d_1,\ldots,d_q)^{\prime}$, 
\begin{equation} \label{eq:multi-B}
(I-\mathcal{B})^{d} Y_t
= \left(\begin{matrix} (1-B)^{d_1} & & 0 \\  & \ddots & \\ 0 & & (1-B)^{d_q} \end{matrix}\right) Y_t,
\end{equation}
and $\{X_t\}$ can be a multivariate ARMA process.

Many time series models can be rewritten into a discrete-time Markov chain. Their stationarity can be checked by verifying the geometric ergodicity, defined below. 

\begin{definition}
	Let $\{X_t, t\in \mathbb{Z}\}$ be a temporal homogeneous Markov chain with state space $ (S, \mathcal{F}) $, where the $ \sigma $-field $ \mathcal{F} $ of subsets of $ S $ is assumed to be countably generated. Let $ P^n(x, A) = P(X_{t+n}\in A | X_t = x)$ be the $n$-step transition function.
	$ \{X_t\} $ is geometrically ergodic if there exists a $ 0<\rho <1 $ such that for every $ x\in S $ and some probability measure $\pi$ on $ \mathcal{F} $
	\begin{equation} \label{eq:geo-erg}
	\rho^{-n} \|P^n(x, \cdot) - \pi \|\rightarrow 0 \text{ as } n\rightarrow \infty,
	\end{equation}
	where $ \|\cdot\| $ denotes the total variation of signed measures on $ \mathcal{F} $. 
\end{definition}

For a geometrically ergodic process $\{X_t\}$, from Harris Theorem \citep{meyn2012markov}, we have that $\gamma_X(k)\sim \rho^k$ as $k\rightarrow \infty$, and it hence has short memory.
Heuristically,
geometric ergodicity implies that the Markov chain converges to its stationary distribution exponentially fast. This means that information in the past is forgotten exponentially fast, since the influence of observation $ X_t = x $ on the distribution of $ X_{t+n} $ vanishes exponentially.

\subsection{Recurrent Network Process}

Consider a general recurrent network with input $\{x^{(t)}\}$, output $\{z^{(t)}\}$ and target sequence $\{y^{(t)}\}$, where $z^{(t)}\in\mathbb{R}^p$ and $y^{(t)}\in\mathbb{R}^p$.
We first assume that the data generating process (DGP) is
\begin{equation}
\label{eq:add}
y\ts{t} = z\ts{t} + \varepsilon\ts{t}, \quad \text{ for } t \in \mathbb{Z},
\end{equation}
where $ \{\varepsilon\ts{t}\} $ is a sequence of independent and identically distributed ($i.i.d.$) random vectors. This additive error assumption corresponds to many frequently used loss functions, which measure the distance between $ y\ts{t} $ and $ z\ts{t} $. Examples include $ l_1 $, $ l_2 $, Huber and quantile loss.

We use the term \textit{network process} to describe the generated target sequence by \eqref{eq:add}, and the term \textit{network} to describe the implemented networks, which might include some compromises and simplifications.

We further assume that there is no exogenous input, i.e. $x^{(t)}=y^{(t-1)}$, since the long-range dependence in an exogenous input will confound the memory properties of $\{y^{(t)}\}$.
We aim to check whether a network can model long memory sequences, and it is then equivalent to verifying whether the corresponding network process has long memory under Definition \ref{def1}.

General hidden states $ s\ts{t} \in \mathbb{R}^q $ are also introduced so that the recurrent network process \eqref{eq:add} can be written into a homogeneous Markov Chain with transition function $\mathcal{M}$:
\begin{equation} \label{eq:general-MC}
\left(\begin{matrix} y\ts{t} \\ s\ts{t} \end{matrix}\right)
= \mathcal{M} \left(y\ts{t-1}, s\ts{t-1} \right)
+ \left(\begin{matrix} \varepsilon\ts{t} \\ 0 \end{matrix}\right) .
\end{equation}
When transition function $ \mathcal{M} $ is linear in $ y\ts{t-1} $ and $ s\ts{t-1} $, the above equation also has the form of
\begin{equation} \label{eq:linear-MC}
\left(\begin{matrix} y\ts{t} \\ s\ts{t} \end{matrix}\right)
= W \left(\begin{matrix} y\ts{t-1} \\ s\ts{t-1} \end{matrix}\right)
+ \left(\begin{matrix} \varepsilon\ts{t} \\ 0 \end{matrix}\right) ,
\end{equation}
where $ W $ is a $ (p+q) $-by-$ (p+q) $ transition matrix.

The first example is the vanilla RNN.  Consider a many-to-many RNN structure for time series prediction problems, and use square loss as an example:
\begin{equation}
\begin{cases}
l\ts{t} = \norm{y\ts{t} - z\ts{t}}^2 \\
z\ts{t} = g(W_{zh} h\ts{t} + b_z) \\
h\ts{t} = \sigma(W_{hh} h\ts{t-1} + W_{hy} y\ts{t-1} + b_h)
\end{cases}, 
\label{proc:RNN}
\end{equation}
for $ t \in \{1, ..., T\}$,
where $ y\ts{0} = 0$, $h\ts{0} = 0$,
$ y\ts{t}, z\ts{t} \in \mathbb{R}^{p}$, 
$  h\ts{t} \in \mathbb{R}^{q}$,
$ g $ is an elementwise output function, 
$ \sigma $ is an elementwise activation function.
Accordingly the RNN process can be defined as
\begin{equation} \label{eq:RNN-MC}
\left(\begin{matrix} y\ts{t} \\ h\ts{t} \end{matrix}\right)
= \mathcal{M}_{\textsc{Rnn}} \left(y\ts{t-1}, h\ts{t-1} \right)
+ \left(\begin{matrix} \varepsilon\ts{t} \\ 0 \end{matrix}\right) 
\end{equation}
where $h^{(t)}$ is the hidden state, and $ \mathcal{M}_{\textsc{Rnn}}(\cdot, \cdot): \mathbb{R}^{p+q} \rightarrow \mathbb{R}^{p+q} $ is a function given by
\begin{equation}
\mathcal{M}_{\textsc{Rnn}} \left(u, v\right) = 
\left(\begin{matrix} g(W_{zh} \sigma(W_{hh} v + W_{hy} u + b_h) + b_z) \\ 
\sigma(W_{hh} v + W_{hy} u + b_h) \end{matrix} \right) ,
\label{eq:MRNN}
\end{equation}
with $ u\in\mathbb{R}^{p} $ and  $ v\in\mathbb{R}^{q} $.

The second example is an LSTM network
\begin{equation}
\begin{cases}
l\ts{t} = \|y\ts{t} - z\ts{t}\|^2 \\
z\ts{t} = g(W_{zh} h\ts{t} + b_z) 
\end{cases}, \text{ for } t \in \{1, ..., T\},
\label{proc:LSTM}
\end{equation}
where the hidden unit calculations involve several gating operations as follows,
\begin{equation} \label{eq:lstmcell}
\begin{cases}
f\ts{t} = \sigma(W_{fh} h\ts{t-1} + W_{fy} y\ts{t-1} + b_f) \\
i\ts{t} = \sigma(W_{ih} h\ts{t-1} + W_{iy} y\ts{t-1} + b_i) \\
o\ts{t} = \sigma(W_{oh} h\ts{t-1} + W_{oy} y\ts{t-1} + b_o) \\
\tilde{c}\ts{t} = \tanh (W_{ch} h\ts{t-1} + W_{cy} y\ts{t-1} + b_c) \\
c\ts{t} = i\ts{t} \odot \tilde{c}\ts{t} + f\ts{t} \odot c\ts{t-1} \\
h\ts{t} = o\ts{t} \odot \tanh(c\ts{t})	
\end{cases}, 
\end{equation}
for $t \in \{1, ..., T\}$,
where $ y\ts{0} = 0$, $ h\ts{0} = 0$,
$ y^{(t)}, z^{(t)} \in \mathbb{R}^p $, 
$ h^{(t)}, c\ts{t}, \tilde{c}\ts{t}, f\ts{t}, i\ts{t}, o\ts{t} \in \mathbb{R}^q $, 
$ g $ is an elementwise output function, 
$ \sigma $ is the elementwise sigmoid function, 
$ \tanh $ is the elementwise hyperbolic tangent function, 
$ \odot $ is the elementwise product.
Similarly the LSTM process can be defined as
\begin{equation} \label{eq:LSTM-MC}
\left(\begin{matrix} y\ts{t} \\ h\ts{t} \\ c\ts{t} \end{matrix}\right)
= \mathcal{M}_{\textsc{Lstm}} \left(y\ts{t-1}, h\ts{t-1}, c\ts{t-1} \right)
+ \left(\begin{matrix} \varepsilon^{(t)} \\ 0 \\ 0 \end{matrix}\right) 
\end{equation}
where $h^{(t)}$ and $c^{(t)}$ are hidden states defined in \eqref{eq:general-MC}, the transition function $  \mathcal{M}_{\textsc{Lstm}}(\cdot, \cdot, \cdot): \mathbb{R}^{p+q+q} \rightarrow \mathbb{R}^{p+q+q} $ has a complicated form, and hence omitted here.

\subsection{Memory Property of Recurrent Network Processes}

We first introduce a technical assumption, and then state two general theorems for recurrent network processes.

\begin{ass} \label{ass:f}
	(i) The joint density function of $\varepsilon\ts{t}$ is continuous and positive everywhere; 
	(ii) For some $ \kappa \geq 2 $, $ E\|\varepsilon\ts{t}\|^{\kappa} < \infty$. 
\end{ass}

\begin{thm} \label{thm:general}	
	Under Assumptions 1, if there exist real numbers $0< a < 1$ and $b$ such that $\norm{\mathcal{M}(x)} \leq a\norm{x} + b$, then recurrent network process (\ref{eq:general-MC}) is geometrically ergodic, and hence has short memory.
\end{thm}

\begin{thm} \label{thm:general2}	
	Under Assumption 1, linear recurrent network process (\ref{eq:linear-MC}) is geometrically ergodic if and only if spectral radius $\rho(W) < 1$. Model (\ref{eq:linear-MC}) hence has short memory.
\end{thm}

Technical proofs of the above two theorems are deferred to the supplementary material. Theorem \ref{thm:general2} gives a sufficient and necessary condition, while Theorem \ref{thm:general} provides only a sufficient condition since it is usually difficult to achieve a sufficient and necessary condition for a nonlinear model \citep{zhu2018linear}.

It is implied by Theorems \ref{thm:general} and \ref{thm:general2} that both RNN and LSTM have no capability of handling a stable time series with long-range dependence.

Specifically we consider a RNN process at \eqref{eq:RNN-MC} with $p=q=1$.
Assume that the norm $ \|\cdot \| $ in Theorem \ref{thm:general} takes $l_1$ norm, the output function $ g $ is linear, sigmoid or softmax, and the activation function $ \sigma $ is ReLU, sigmoid or hyperbolic tangent.
Table \ref{tab:RNN} gives the ranges of weights so that the RNN process is geometrically ergodic with short memory.

For linear or ReLU activation, RNN has short memory when the magnitude of the weights is bounded away from one. This condition is often satisfied for numerically stable RNNs. Thus, RNN with linear or ReLU activation often has short memory. 
Moreover, in practice, the activation function in RNN is commonly chosen as $\tanh$.
According to Table \ref{tab:RNN}, RNN with $ \tanh $ or sigmoid activation always has short memory.
In fact, this holds for general RNN processes with any bounded and continuous output and activation function; see Corollary \ref{cor:RNN-W} below.

\begin{cor} \label{cor:RNN-W}
	Suppose that the output and activation functions, $g(\cdot)$ and $\sigma(\cdot)$, at \eqref{eq:RNN-MC} are continuous and bounded. If Assumption \ref{ass:f} holds, then the RNN process is geometrically ergodic and has short memory.
\end{cor}

\begin{table}[t]
	\caption{Restrictions on weights such that the RNN process is geometrically ergodic.}
	\label{tab:RNN}
	\begin{center}
		\begin{small}
			\begin{tabular}{ccc}
				\toprule
				\multirow{2}{*}{\makecell{Output\\ function $g$}} & \multicolumn{2}{c}{Activation function $ \sigma $} \\
				\cmidrule[0.5pt](lr{0.125em}){2-3}%
				& {identity} or {ReLU} & {sigmoid} or {tanh}  \\
				\hline
				{identity} & {\makecell{$|w_{zh}w_{hh}|\leq a$, \\$|w_{zh}w_{hy}|\leq a$,\\$|w_{hh}|\leq a,|w_{hy}|\leq a$}}
				& No \\
				\midrule 
				{sigmoid} & {$|w_{hh}|\leq a,|w_{hy}|\leq a$} & No \\
				\midrule 
				{softmax} & {$|w_{hh}|\leq a,|w_{hy}|\leq a$} & No \\ 
				\bottomrule
			\end{tabular}
		\end{small}
	\end{center}
	\vskip -0.1in
\end{table}

We next consider the LSTM process at \eqref{eq:LSTM-MC}, and 
a detailed analysis of one-dimensional LSTM, similar to the analysis presented in Table \ref{tab:RNN}, can be referred to in the supplementary material.
The restrictions on weights such that an LSTM process is geometrically ergodic is given below.

\begin{cor} \label{cor:LSTM-W}
	The input series features $\{y^{(t-1)}\}$ are scaled to the range of $[-1, 1]$. Suppose that $ M:= \sup_{x\in B_\infty^q} \|g(W_{zh}x + b_z)\|_{l_1} <\infty $ and $ \sigma(\|W_{fh}\|_{l_\infty} + \|W_{fy}\|_{l_\infty} + \|b_f\|_{l_\infty}) \leq a$ for some $ a < 1 $, where $ B_\infty^q $ is the $ q $-dimensional $ l_\infty $-ball and $ \|W\|_{l_\infty} = \max_{1\leq i\leq m} \sum_{j = 1}^{n} |w_{ij}| $ is the matrix $ l_\infty $-norm.
	If Assumption \ref{ass:f} holds, then the LSTM process at \eqref{eq:LSTM-MC} is geometrically ergodic and has short memory.
\end{cor}

It is worthy to point out that the condition of $ M <\infty $ is satisfied by many commonly used output functions, including linear function, ReLU, sigmoid or $\tanh$. 
From the conditions in Corollary \ref{cor:LSTM-W}, we may conclude that the forget gate mainly affects the memory property of an LSTM.

\subsection{Long Memory Network Process} \label{sec:lmnet}

In the previous subsection, we verify that the two commonly used recurrent networks, RNN and LSTM, are both short memory when there is no exogenous input. 
However, it is very common of $\{x^{(t)}\}$ to include exogenous inputs, and the long-range dependence in $ \{x\ts{t} \} $ will confound the memory property of $ \{y\ts{t} \} $. It is then misleading to check the long-range dependence of a network by verifying that of the corresponding network process.

In the literature of conditional heteroscedastic time series models, there also exists vast interest in long-range dependence; see, e.g., the fractionally integrated GARCH (FIGARCH) models \citep{davidson2004moment}, hyperbolic GARCH (HYGARCH) models \citep{li2015new}, etc.
However, for a stationary conditional heteroscedastic process, its \vphantom{squared }autocovariance function is always summable, i.e. it always has short memory in terms of Definition \ref{def1}.

Consider a stationary conditional heteroscedastic model, $y_t=\eta_t\sqrt{h_t}$ and $ h_t = \omega + \sum_{k=1}^{\infty} \theta_k y_{t-k}^2 $, where $\{\eta_t\}$ is a sequence of $i.i.d.$ random variables.
\citet{davidson2004moment} defined that it has long memory if the coefficients $\{\theta_k\}$ have hyperbolic decay.
This definition has been widely used in the literature. Moreover, the commonly used fractional ARIMA process also has the form of 
$ y\ts{t} = \sum_{k=1}^{\infty} a_k y\ts{t-k} + \varepsilon\ts{t} $, where $ a_k \sim k^{-d-1} $ as $k\rightarrow\infty$.
This motivates us to propose a new definition of long memory for a network process, 
\begin{equation}\label{longmemoryprocess}
y\ts{t} = \sum_{k=0}^{\infty} A_k x\ts{t-k} + \varepsilon\ts{t},
\end{equation}
where $y^{(t)}\in\mathbb{R}^p$ is the generated target, $x^{(t)}\in\mathbb{R}^q$ is the input, $q\times p$ coefficient matrices $A_k=\{(A_k)_{ij}\}$ with $1\leq i \leq p$ and $1\leq j\leq q$ and $\{\varepsilon^{(t)}\}$ are $i.i.d.$ random vectors.

\begin{definition} \label{def:mem-net}
	The network process at \eqref{longmemoryprocess} has long memory if there exist $1\leq i \leq p$ and $1\leq j\leq q$ such that
	\begin{equation*}
	(A_k)_{ij} \sim k^{-d-1}
	\end{equation*}
	as $k\rightarrow\infty$, where $d \in (0, 0.5)$ is the memory parameter.
\end{definition}

The coefficients $ \{A_k\}_{k\in \mathbb{N}} $ reflect the dependence of $ y\ts{t} $ on $ \{ x\ts{t} \} $. 
If $ (A_k)_{ij} $ decays slowly at a polynomial rate, we can conclude that $ y\ts{t}_i $ is long-term dependent on $ x\ts{t}_j $ according to Definition \ref{def1}.
This definition is consistent with the visual method used by Lin et al. \yrcite{lin1996learning} since the Jacobian $ J(t, k) = {\partial y\ts{t}}/{\partial x\ts{t-k}} $ equals $ A_k $ and exhibits slow decay under Definition \ref{def:mem-net}.

This definition admits extensions to nonlinear networks. 
Firstly, we can linearize a network by letting all output and activation functions be the identity function, and we call the resulting network process as its \textit{linear network process}. 
A network can handle data with long-range dependence if its linear network process has long memory in terms of Definition \ref{def:mem-net}. 
Note that, for some networks such as LSTM, their linear network processes are still nonlinear. 
Moreover, nonlinear networks can often be well approximated by linear networks via polynomial approximations and introducing the powers of $ x\ts{t} $ into the input; see Yu et al. \yrcite{yu2017long} as an example. 
Lastly, we may resort to a first-order approximation $ y\ts{t} \approx \sum_{k=0}^{\infty} J(t, k) \, x\ts{t-k} + \varepsilon\ts{t} $ as in Lin et al. \yrcite{lin1996learning} and check the decay pattern of the Jacobians.

\section{Long Memory Recurrent Networks}
Motivated by the fractionally integrated process at \eqref{eq:multi-B} and the new definition of long memory in Section 2.4, we propose a long memory filter structure that can be added to neural networks to enable modeling long-term dependence. 
This long memory filter can be viewed as a special attention mechanism with only a few memory parameters and a guaranteed memory elongation effect when active.
In Sections \ref{sec:MRNN} and \ref{sec:MLSTM}, we modify the RNN and the LSTM network by the proposed long memory filter structure, respectively.

\subsection{Memory-augmented RNN (MRNN)} \label{sec:MRNN}

Long memory pattern is introduced via a memory filter, 
\begin{equation}\label{filter}
F(x\ts{t};d)=((I-\mathcal{B})^{d}-I) x\ts{t}
\end{equation}
for $x^{(t)}\in\mathbb{R}^{p_x}$ with $p_x$ being the dimension of inputs, where $d=(d_1,\ldots,d_{p_x})^{\prime}$; see \eqref{eq:multi-B} for details.
The memory filter can be viewed as soft attention to a reasonably sized memory with only a few memory parameters of $d_i$s.
Note that, from \eqref{eq:1-Bd}, the $i$th element of the memory filter is $ F(x\ts{t};d)_i = ((1-B)^{d_i}-1) x_i\ts{t} = \sum_{j=1}^{\infty} w_j(d_i) x\ts{t-j+1}_i $, where $ w_j(d) = {\Gamma(d+j)}/{[j!\Gamma(d)]} = \prod_{i=0}^{j-1} ({i-d})/({i+1})$.
To implement this model, we truncate this infinite summation at lag $ K$.
In our experiments, we choose $ K = 100$ by default, since, taking $d=0.4$ as an example, 
$ w_{100}(0.4) \approx -4.27 \times 10^{-4}$ 
is already small enough.

\begin{figure}[ht]
	\centering
	\includegraphics[width=\columnwidth]{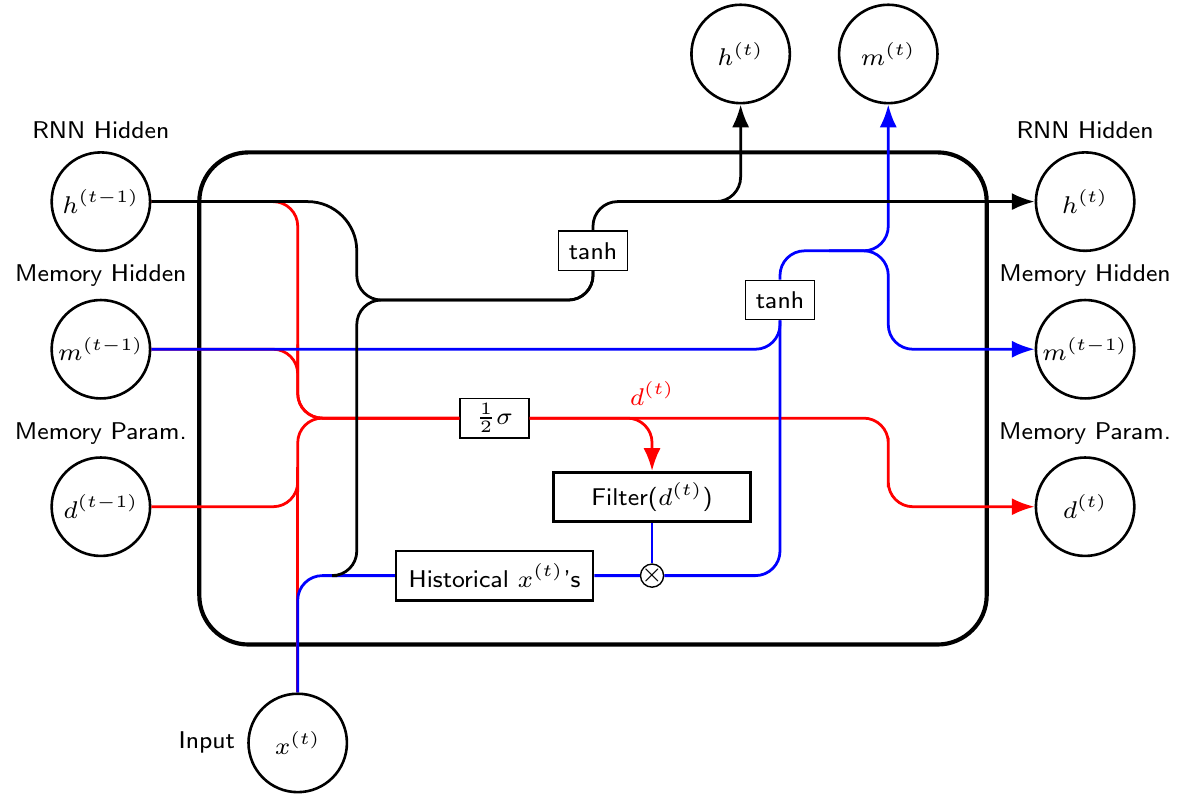}
	\vskip -0.05 in
	\caption{The MRNN cell structure.}
	\label{fig:MRNNcell}
\end{figure}

In MRNN, we introduce a long memory hidden unit 
$$ m\ts{t} = \tanh(W_{mm} m\ts{t-1} + W_{mf} F(x\ts{t};d) + b\ts{m}),$$ 
which shares a similar structure with $ h\ts{t} $, while the input term is replaced with the filter $ F(x\ts{t};d) $.
The unit $ m\ts{t} $ works in parallel with the traditional hidden unit $ h\ts{t} $ and takes responsibility for modeling the long memory pattern in time series.
The underlying network can then be designed,
\begin{equation*} \label{eq:MRNN-DGP}
z\ts{t}  = g(W_{zh} h\ts{t} + W_{zm} m\ts{t} + b_z) + \varepsilon\ts{t}, \text{ \quad for } t \in \mathbb{Z} 
\end{equation*}
where $ z\ts{t} \in \mathbb{R}^{p_z} $ with $p_z$ being the dimension of outputs, and
$ h\ts{t} = (h_1\ts{t}, \cdots, h_q\ts{t})' \in \mathbb{R}^q $ is the same as the hidden states in RNN. 

For the memory parameter $d=(d_1,\ldots,d_{p_x})^{\prime}$, we restrict $0<d_i<0.5$ such that the fractional integration can provide long memory. Moreover, to make the design more general, we also let the memory parameter $d$ depend on other variables, and hence the notation $d^{(t)}$.
As a result, the procedure of MRNN is given below.
\begin{equation}
\begin{cases}
l\ts{t} = \|y\ts{t} - z\ts{t}\|^2 \\
z\ts{t} = g(W_{zh} h\ts{t} + W_{zm} m\ts{t} + b_z) \\
h\ts{t} = \tanh(W_{hh} h\ts{t-1} + W_{hx} x\ts{t} + b_h) \\
F(x\ts{t};d\ts{t})_i = \sum_{j=1}^{K} w_j(d_i\ts{t}) x_i\ts{t-j+1} \\
d\ts{t} = \frac{1}{2} \sigma( W_{d} [d\ts{t-1}, h\ts{t-1}, m\ts{t-1}, x\ts{t}]+ b_d) \\
m\ts{t} = \tanh(W_{m} [m\ts{t-1}, F(x\ts{t};d\ts{t})] + b_m)
\end{cases}
\label{proc:mrnn}
\end{equation}
for $ i\in \{1, ..., p_x\}, t \in \{1, ..., T\}$, where $h\ts{t}, m\ts{t} \in \mathbb{R}^q$, $d\ts{t}, x\ts{t} \in \mathbb{R}^{p_x}$, and $ \sigma $ is the sigmoid activation function.


Figure \ref{fig:MRNNcell} gives a graphical representation of the MRNN cell structure, imitating the style of Olah \yrcite{olah2015understanding}.
The memory filter \textit{Filter$(d\ts{t})$} refers to \eqref{filter}, and \textit{Historical $ x\ts{t}$'s} are $ (x\ts{t}, x\ts{t-1}, ..., x\ts{t-K+1}) $, where we treat $ x\ts{s} = 0 $ for $ s\leq 0 $. 

For the network with memory parameters constant over time points $t$, we refer to it as the MRNNF model, and it can be implemented by fixing $ W_{d} = 0 $ in the update equation \eqref{proc:mrnn}.

\begin{thm} \label{thm:MRNN}
	In terms of Definition \ref{def:mem-net}, the MRNNF has the capability of handling long-range dependence data, while the RNN cannot.
\end{thm}

Technical proof is provided in the supplementary material.

\subsection{Memory-augmented LSTM (MLSTM)} \label{sec:MLSTM}

From Corollary \ref{cor:LSTM-W}, we know that the forget gates determine the memory property of the original LSTM. The update equation of cell states has a varying coefficient vector AR(1) form, 
\[
c\ts{t}-f\ts{t} \odot c\ts{t-1} = i\ts{t} \odot \tilde{c}\ts{t}, 
\]
which has short memory when the coefficients $f^{(t)}$s are smaller than one. 
Thus, we propose to revise the cell states of LSTM by adding the long memory filter to it
\[
(I-\mathcal{B})^{d} c\ts{t} = i\ts{t} \odot \tilde{c}\ts{t},
\]
where the memory parameter $d$ can depend on other variables as for the MRNN in the previous subsection.
The MLSTM cell structure is shown in Figure \ref{fig:MLSTMcell}. The revised cell states can be viewed as paying soft attention to past cell states controlled by only a few memory parameters.

\begin{figure}[ht]
	\centering
	\includegraphics[width=\columnwidth]{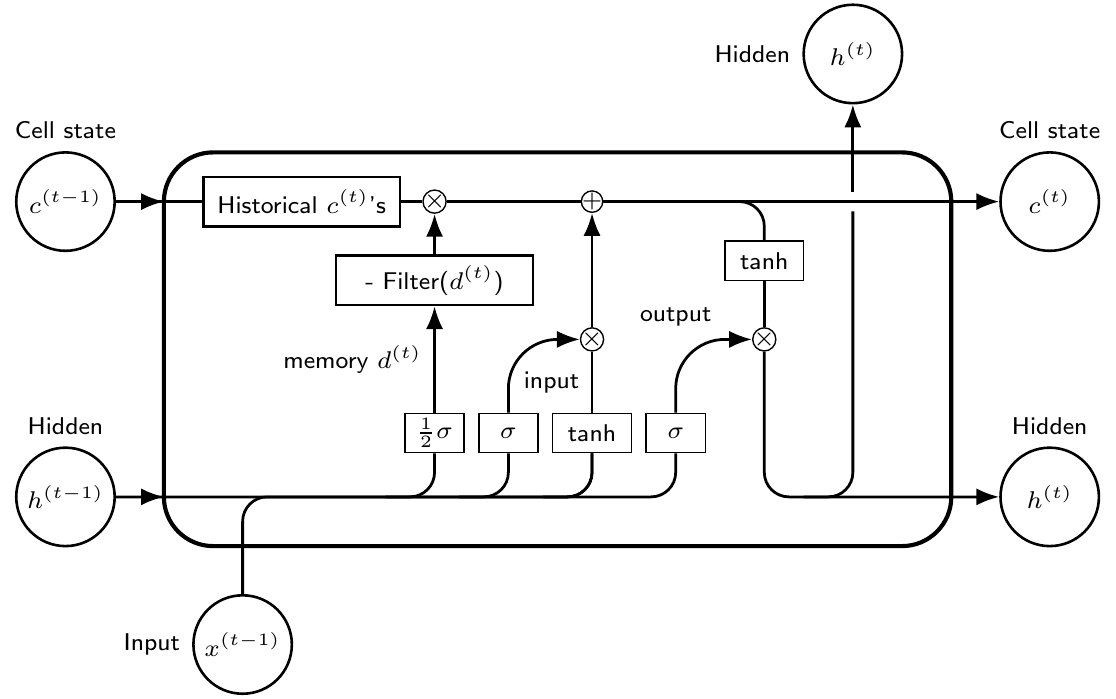}
	\vskip -0.05 in
	\caption{The MLSTM cell structure.}
	\label{fig:MLSTMcell}
\end{figure}

The underlying network can then be designed as
\begin{equation} \label{eq:MLSTM-DGP}
z\ts{t}  = g(W_{zh} h\ts{t} + b_z) + \varepsilon\ts{t}, \text{ \quad for } t \in \mathbb{Z},
\end{equation}
where the hidden unit $ h\ts{t} $ is produced by the MLSTM cell. 
We rename the forget gate as the memory gate
and modify the update equations as
\begin{equation} \label{eq:mlstmcell}
\begin{cases}
d\ts{t} = \frac{1}{2}\sigma(W_{d} [d\ts{t-1}, h\ts{t-1}, x\ts{t}] + b_d) \\
i\ts{t} = \sigma(W_{ih} h\ts{t-1} + W_{ix} x\ts{t} + b_i) \\
o\ts{t} = \sigma(W_{oh} h\ts{t-1} + W_{ox} x\ts{t} + b_o) \\
\tilde{c}\ts{t} = \tanh (W_{ch} h\ts{t-1} + W_{cx} x\ts{t} + b_c) \\
(I-\mathcal{B})^{d} c\ts{t} = i\ts{t} \odot \tilde{c}\ts{t} \hspace{2mm}\text{with}\hspace{2mm} d=d\ts{t}\\
h\ts{t} = o\ts{t} \odot \tanh(c\ts{t})	
\end{cases}, 
\end{equation}
for $t \in \{1, ..., T\}$,
where 
$ (I-\mathcal{B})^{d} $ is defined as in \eqref{eq:multi-B}.

As for the MLSTM, to implement this model, we truncate the infinite summation $ (I-\mathcal{B})^{d} c\ts{t} $ at lag $ K $. The procedure related to the MLSTM cell states is implemented as 
\begin{equation} \label{proc:MLSTM-CS}
c_i\ts{t} = - \sum_{j=1}^{K} w_j(d_i\ts{t}) c_i\ts{t-j} + i\ts{t} \tilde{c}\ts{t},
\end{equation}
where $c_i\ts{t}$ is the $i$-th element of $c\ts{t}$, $ w_j(d) = {\Gamma(d+j)}/[{j!\Gamma(d)}] = \prod_{i=0}^{j-1} ({i-d})/({i+1})$.
We remark that the negative sign before the summation in equation \eqref{proc:MLSTM-CS} is introduced by the definition of $ (I - \mathcal{B})^{d} $. In MRNN, we let the negative sign be absorbed by the weight matrix $ W_{mf} $ for elegance.

Note that the cell state $\{c^{(t)}\}$ has long memory in terms of Definition \ref{def:mem-net}. In the meanwhile, due to the gating mechanism, neither LSTM nor MLSTM can be reasonably simplified to a linear network, i.e. their linear network processes are both nonlinear. However, if we assume that the gates are learnable constants independent of the hidden unit and the inputs, we can prove that the constant-gates-LSTM does not have long memory, while the constant-gates-MLSTM is a long memory network according to Definition \ref{def:mem-net}. We defer the formal statement of this auxiliary result and its proof to the supplementary material. 

Similar to MRNNF, we refer MLSTMF to the case with constant memory parameter over time $t$, and it can be implemented via fixing $ W_{d} = 0 $ in the update equation \eqref{eq:mlstmcell}.

\section{Experiments}
This section reports several numerical experiments. 
We first compare the models using time series forecasting tasks on four long memory datasets and one short memory dataset. 
Then, we investigate the effect of the model parameter $K$ on the forecasting performance.
Lastly, we apply the proposed models to two sentiment analysis tasks.

All the networks are implemented in \texttt{PyTorch}. 
We use the Adam algorithm with learning rate $ 0.01 $ for optimization. 
The optimization is stopped when the loss function drops by less than $ 10^{-5} $ or has been increasing for 100 steps or has reached 1000 steps in total.
The learned model is chosen to be the one with the smallest loss on the validation set.

Considering the non-convexity of the optimization, we initialize with 100 different random seeds and arrive at 100 different trained models for each model. 
We refer to the distribution of these 100 results as the \textit{overall performance}, and best of them as the \textit{best performance}. 
The overall performance reflects what we can expect from a locally optimal model, and the best performance is closer to the outcome of a globally optimal model.

\subsection{Long Memory Datasets} \label{sec:exp-long}

We compare our models with the baselines on one synthetic dataset and three real datasets. 
We split the datasets into training, validation and test sets, and report their lengths below using notation $ (n_{train} + n_{val} + n_{test})$. 
MSE is the target loss function for training.
We perform one-step rolling forecasts on the test set and calculate prediction RMSE, MAE, and MAPE.

\textbf{ARFIMA series} 
We generated a series of length $4001$ $(2000+1200+800)$ using the model $ (1-0.7B + 0.4B^2)(1-B)^{0.4} Y_t = (1 - 0.2B) \varepsilon_t $ with obvious long memory effect.

\textbf{Dow Jones Industrial Average (DJI)} 
The raw dataset contains DJI daily closing prices from 2000 to 2019 obtained from Yahoo Finance. We convert it to absolute log return for $5030$ $(2500+1500+1029)$ days in order to model the long memory effect in volatility.

\textbf{Metro interstate traffic volume} 
The raw dataset contains hourly Interstate 94 Westbound traffic volume for MN DoT ATR station 301, roughly midway between Minneapolis and St Paul, MN, obtained from MN Department of Transportation \cite{UCItraffic}. We convert it to de-seasoned daily data with length $1860$ $(1400+200+259)$.

\textbf{Tree ring} 
Dataset contains $4351$ $(2500+1000+850)$ tree ring measures of a pine from Indian Garden, Nevada Gt Basin obtained from R package \texttt{tsdl} \cite{tsdl}.

We visualize the long memory in the datasets via autocorrelation plots (Figure \ref{fig:acfs}), where $ r_0 = 1 $ is cropped-off. 
Full ACF plots are provided in the supplementary material.

\begin{figure}[ht]
	\centering
	\includegraphics[width=\columnwidth]{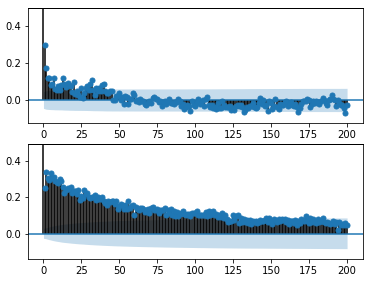}
	\vskip -0.10 in
	\caption{Autocorrelation plot of traffic dataset (top) and DJI dataset (bottom).}
	\label{fig:acfs}
\end{figure}

We compare the following models: 
0. ARFIMA;
1. vanilla RNN (RNN);
2. two-lane RNN with past $K$ values as input (RNN2);
3. Recurrent weighted average network (RWA);
4. MIxed hiSTory RNNs (MIST);
5. MRNN with homogeneous memory parameter $ d $ (MRNNF);
6. MRNN with dynamic $ d\ts{t} $ (MRNN);
7. vanilla LSTM (LSTM);
8. MLSTM with homogeneous $ d $ (MLSTMF); and
9. MLSTM with dynamic $ d\ts{t} $ (MLSTM).

\begin{table}[t]
	\caption{Overall performance in terms of RMSE. Average RMSE and the standard deviation (in brackets) are reported. The best result is highlighted in \textbf{bold}.}
	\label{tab:average}
	\vskip 0.15in
	\begin{center}
		\begin{small}
			\begin{tabular}{lcccc}
				\toprule
				& ARFIMA & DJI (x100)& Traffic & Tree \\
				\cmidrule[0.5pt]{1-5}
				RNN 	& \makecell{1.1620\\ (0.1980)} & \makecell{0.2605\\ (0.0171)} & \makecell{336.44\\ (10.401)} & \makecell{0.2871\\ (0.0086)} \\
				RNN2 	& \makecell{1.1630\\ (0.1820)} & \makecell{0.2521\\ (0.0112)} & \makecell{336.32\\ (10.182)} & \makecell{0.2855\\ (0.0077)}\\
				RWA 	& \makecell{1.6840\\ (0.0050)} 
				& \makecell{0.2689\\ (0.0095)} & \makecell{346.62\\ ({1.410})} & \makecell{0.3048\\ (0.0001)}\\
				MIST 	& \makecell{1.1390\\ (0.1832)} & \makecell{{0.2604}\\ (0.0154)} & \makecell{{358.09}\\ (16.270)} & \makecell{0.2883\\ (0.0091)}\\
				MRNNF 	& \makecell{1.1010\\ (0.1000)} & \makecell{\textbf{0.2472}\\ (0.0109)} & \makecell{\textbf{333.36}\\ (8.453)} & \makecell{0.2822\\ (0.0048)}\\
				MRNN 	& \makecell{\textbf{1.0880}\\ (0.1140)} 
				& \makecell{0.2487\\ (0.0105)} & \makecell{333.72\\ (10.157)} & \makecell{\textbf{0.2818}\\ (0.0053)}\\
				\cmidrule[0.5pt]{1-5}
				LSTM 	& \makecell{1.1340\\ (0.1200)} & \makecell{0.2492\\ (0.0128)} & \makecell{337.60\\ (8.146)} & \makecell{0.2833\\ (0.0070)}\\
				MLSTMF 	& \makecell{1.1580\\ (0.1660)} & \makecell{0.2540\\ (0.0139)} & \makecell{337.78\\ (9.020)} & \makecell{0.2859\\ (0.0082)}\\
				MLSTM 	& \makecell{1.1490\\ (0.1660)} & \makecell{0.2531\\ (0.0130)} & \makecell{337.83\\ (9.440)} & \makecell{0.2859\\ (0.0083)}\\
				\bottomrule 
			\end{tabular}
		\end{small}
	\end{center}
	\vskip -0.15in
\end{table}

\begin{figure}[ht]
	\centering
	\includegraphics[width=0.95\columnwidth]{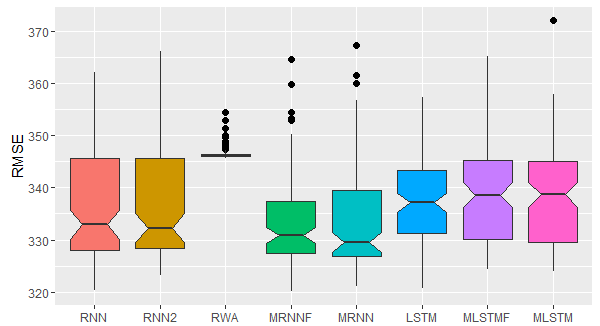}
	\caption{Boxplot of RMSE for 100 different initializations. Dataset: traffic.}
	\label{fig:traffic-RMSE-box}
\end{figure}

In Table \ref{tab:average}, we report overall performance of one-step forecasting regarding RMSE. MAE and MAPE are reported in the supplementary material.
Boxplots are generated to give a better picture for comparison. We use the traffic dataset as an example and put boxplots for other datasets into the supplementary material. 
We can see that MRNN and MRNNF have a smaller average RMSE and smaller quantiles compared with others. 
MLSTM(F) does not have an obvious advantage over LSTM in terms of RMSE, and we suspect that this is due to the difficulty in training MLSTM(F). 
RWA is very stable with respect to initialization but does not have a competitive RMSE.
The \texttt{arfima} routine in \texttt{R} automatically searches for an optimal global solution, and thus ARFIMA appears only in the comparison of the best performance.

Thanks to the anonymous reviewers' comments, we present two-sample $t$-tests to compare the models more rigorously. Consider the null and alternative hypotheses
$ H_0 $: mean(RMSE(Model)) $\geq$ mean(RMSE(Benchmark)) vs. 
$ H_1 $: mean(RMSE(Model)) $<$ mean(RMSE(Benchmark)).
MRNN is significantly better than RNN at 5\% significance level on all datasets, and it is significantly better than LSTM on all datasets except for DJI.

In Table \ref{tab:best}, we report best performance of one-step forecasting regarding RMSE. Results in terms of MAE and MAPE are reported in the supplementary material. 
We can see that the best performance of MRNNF and MRNN are better than others on ARFIMA, traffic and tree datasets, while remains competent on DJI. 

\begin{table}[t]
	\caption{Best performance in terms of RMSE.}
	\label{tab:best}
	\vskip 0.15in
	\centering
	\begin{small}
		\begin{tabular}{lcccc}
			\toprule
			& ARFIMA & DJI (x100)& Traffic & Tree \\
			\cmidrule[0.5pt]{1-5} 
			ARFIMA 	& 1.0260 & 0.2468 	 & 327.47  & 0.2773 \\
			\cmidrule[0.5pt]{1-5}
			RNN 	& 1.0452 & \textbf{0.2390} & 320.29  & 0.2786 \\
			RNN2 	& 1.0232 & 0.2402	 & 323.15  & 0.2784 \\
			RWA 	& 1.6742 & 0.2631	 & 345.58  & 0.3047 \\
			MIST    & 1.0232 & 0.2401 	 & 337.49  & 0.2772 \\
			MRNNF 	& 1.0230 & 0.2394 	 & \textbf{320.09} & \textbf{0.2769}\\
			MRNN 	& \textbf{1.0208} 	 & 0.2395 & 321.03  & 0.2770 \\
			\cmidrule[0.5pt]{1-5} 
			LSTM 	& 1.0272 & 0.2396	 & 320.79  & 0.2771 \\
			MLSTMF	& 1.0280 & 0.2413	 & 324.37  & 0.2773 \\
			MLSTM 	& 1.0272 & 0.2412	 & 324.00  & 0.2772 \\
			\bottomrule 
		\end{tabular}
	\end{small}
	\vskip -0.15in
\end{table}

\subsection{Short Memory Dataset}
For datasets without long memory effect or with long memory only in certain dimensions, the performance of our proposed models does not deteriorate. 
This claim is supported by an experiment on a synthetic dataset generated by RNN.

We generated a sequence of length 4001 ($2000+1200+800$) using model \eqref{eq:RNN-MC}, which does not have long memory according to Corollary \ref{cor:RNN-W}. We refer to this synthetic dataset as the RNN dataset. The boxplots of error measures are presented in the supplementary material. From the boxplots, we can see that the performance of MRNN(F) and MLSTM(F) is comparable with that of the true model RNN, except that the variation of the error measures is a bit larger.

\subsection{Model Parameter $K$}
We further explore more choices of $K$ using the long memory datasets in section \ref{sec:exp-long}. 
Settings for MSLTM and LSTM are kept the same except for the hyperparameter $K$.
We compare the prediction performance of the proposed models with $ K = 25, 50, 75$ or $100$. 
For MRNN and MRNNF, the prediction is generally better for a larger $K$, and they have smaller average RMSE than all the baseline models regardless of the choice of $K$. Interestingly, the performance of MLSTM and MLSTMF gets better when $K$ becomes smaller, and with $K=25$, they can outperform LSTM on ARFIMA and traffic datasets. Thus, we recommend a large $K$ for MRNN and MRNNF models, while for the more complicated MLSTM models, $K$ deserves more investigation to balance expressiveness and optimization. Detailed results and more figures can be found in the supplementary material.

\subsection{Sentiment Analysis}
As suggested by the reviewers, we present two more applications of our proposed model on natural language processing tasks. Comparisons between our proposed model and RNN/LSTM are made on two sentiment analysis datasets, CMU-MOSI \citep{zadeh2016mosi, zadeh2016multimodal, zadeh2018multi} and a paper reviews dataset \cite{keith2017hybrid}.
For faster computation, we fix the memory parameter $d$ to be homogeneous and decrease $K$ to $50$ in MRNN and MLSTM.

CMU-MOSI contains acoustic, language and visual information from videos.
Each sample in CMU-MOSI is annotated with a value ranging from -3 to 3.
A larger annotation indicates more positive sentiment.
The models are all trained using the MAE loss, and the overall performance is reported in Table \ref{tab:cmu-mosi}.
We conduct the same two-sample $t$-tests for MRNNF50 against RNN/LSTM using measure MAE, and the $p$-values are 0.004 and 0.156. 
Corresponding $p$-values for MLSTMF50 are 0.005 and 0.349.

\begin{table}[t]
	\caption{Overall performance on CMU-MOSI in terms of MAE, RMSE and MAPE. } 
	\label{tab:cmu-mosi}
	\vskip 0.10 in
	\begin{center}
		\begin{small}
			\begin{tabular}{lccc}
				\toprule
				& MAE & RMSE & MAPE \\
				\cmidrule[0.5pt]{1-4}
				RNN 	& \makecell{1.5028\\ (0.0186)} & \makecell{1.7368\\ (0.0171)} & \makecell{1.0314\\ (0.0339)}  \\
				LSTM 	& \makecell{1.4978\\ (0.0128)} & \makecell{1.7288\\ (0.0112)} & \makecell{\textbf{1.0146}\\ (0.0186)} \\
				MRNNF50	& \makecell{\textbf{1.4953}\\ (0.0216)} & \makecell{\textbf{1.7255}\\ (0.0109)} & \makecell{{1.0322}\\ (0.0351)} \\
				MLSTMF50	& \makecell{1.4972 \\ (0.0108)} & \makecell{1.7279\\ (0.0105)} & \makecell{1.0156\\ (0.0110)} \\
				\bottomrule 
			\end{tabular}
		\end{small}
	\end{center}
	\vskip -0.25in
\end{table}

The Paper Reviews dataset \cite{keith2017hybrid} contains 405 textual reviews evaluated with a 5-point scale. In preprocessing, we removed empty reviews and English reviews, leading to 382 remaining instances. 
For simplicity, we use a 2-layer network structure with a fully connected classifier at the output. 
The first layer uses RNN, LSTM, MRNNF50 or MLSTMF50, and the 2nd layer is fixed to be LSTM.

The overall performance is reported in Table \ref{tab:reviews} and the best performance is reported in Table \ref{tab:reviews-best}.
Using MLSTMF50 as the first layer leads to a significant improvement in all measures over LSTM.
For the best performance, MRNNF50 achieves the highest accuracy, while MLSTMF50 has clear-cut advantages on all other metrics.
Considering accuracy, the $p$-values for MLSTMF50 against RNN/LSTM are $<0.001$ and $0.040$, and that for MRNNF are $<0.001$ and $0.105$.
This experiment indicates that our proposed network component can be combined with existing ones to improve performance. 

\begin{table}[t]
	\caption{Overall performance on Paper Reviews in terms of accuracy, precision, recall and cross-entropy loss (CEloss). } 
	\label{tab:reviews}
	\begin{center}
		\begin{small}
			\begin{tabular}{lcccc}
				\toprule
				& Accuracy & Precision & Recall & CEloss\\
				\cmidrule[0.5pt]{1-5}
				RNN 	& \makecell{0.2836\\ (0.0348)} & \makecell{0.1786\\ (0.0606)} & \makecell{0.2248\\ (0.0350)}  & \makecell{{1.5787}\\ (0.0348)}\\
				LSTM 	& \makecell{0.3021\\ (0.0468)} & \makecell{0.1724\\ (0.0697)} & \makecell{{0.2274}\\ (0.0332)}& \makecell{{1.5752}\\ (0.0189)}\\
				MRNNF50	& \makecell{0.3096\\ (0.0373)} & \makecell{0.1692\\ (0.0839)} & \makecell{{0.2224}\\ (0.0428)}& \makecell{{1.5704}\\ (0.0328)}\\
				MLSTMF50& \makecell{\textbf{0.3110} \\ (0.0204)} 
				& \makecell{\textbf{0.2254}\\ (0.0707)} & \makecell{\textbf{0.2594}\\ (0.0262)} & \makecell{\textbf{1.4758} \\ (0.0218)} \\
				\bottomrule 
			\end{tabular}
		\end{small}
	\end{center}
	\vskip -0.15in
\end{table}

\begin{table}[t]
	\caption{Best performance of the models on Paper Reviews.}
	\label{tab:reviews-best}
	\vskip 0.10 in
	\centering
	\begin{small}
		\begin{tabular}{lcccc}
			\toprule
			& Accuracy & Precision & Recall & CEloss\\
			\cmidrule[0.5pt]{1-5}
			RNN 	& 0.3600 		& 0.3951 	& 0.3093  	& 1.5204 \\
			LSTM 	& 0.3800 		& 0.4304	& 0.3225  	& 1.5512 \\
			MRNNF50 &\textbf{0.4000}& 0.3992 	& 0.3178 	& 1.5209 \\
			MLSTMF50& 0.3600 & \textbf{0.4621} & \textbf{0.3596} & \textbf{1.4489} \\
			\bottomrule 
		\end{tabular}
	\end{small}
	\vskip -0.15in
\end{table}

\section{Conclusion \label{sec:conclude}}
This paper first proves that RNN and LSTM do not have long memory from a time series perspective. By getting use of fractionally integrated processes, we propose the corresponding modifications such that they can handle the long-range dependence data.
MRNN and MRNNF are shown to have advantages in forecasting time series with long-term dependency, 
and a combination of MLSTMF50 and LSTM layers significantly improves over a pure LSTM network on a paper reviews dataset.

In terms of future work, it is interesting to know whether the memory filter can bring similar advantages to other variants of recurrent networks or feed-forward networks for sequence modeling. 
Moreover, MRNN and MLSTM with dynamic $d$ is time consuming compared with other models, we leave model simplification and exploring faster optimization approaches to future work.
In term of filter design, by Definition \ref{def:mem-net}, many other slow decaying patterns can also be explored to model long memory sequences. For example, we may let $w_j(d)= j^{-d-1} $ directly. 
Last but not least, currently we only learn a best filter, and it is inspiring to extract long memory via filter banks as in signal processing.

\section*{Acknowledgements}
We thank the anonymous reviewers for their constructive feedback.
This project is sponsored by Huawei Innovation Research Program (HIRP).



\bibliography{LSTM}
\bibliographystyle{icml2020}

\newpage

\appendix
\section{Detailed Theoretical Results}
\subsection{Proof of Theorem \ref{thm:general}} \label{sec:appendix-thm1}

\begin{proof}
	Let $Y\ts{t} = (y\ts{t}{'},s\ts{t}{'}){'}$ and $r = p+q$. Rewrite model \eqref{eq:general-MC} as
	\begin{equation} \label{eq0-1}
		Y\ts{t} = \mathcal{M}(Y\ts{t-1}) + e\ts{t},
	\end{equation}
	where $Y\ts{t}, e\ts{t} \in\mathbb{R}^{r}$ and $  \mathcal{M}: \mathbb{R}^r \rightarrow \mathbb{R}^r $ is a general nonlinear function. 
	
	Let $\mathcal{B}^{r}$ be the class of Borel sets of $\mathbb{R}^{r}$ and $\nu_{r}$ be the Lebesgue measure on $(\mathbb{R}^{r},\mathcal{B}^{r})$. 
	Then, $\{Y\ts{t}\}$ is a homogeneous Markov chain on the state space $(\mathbb{R}^{r},\mathcal{B}^{r},\nu_{r})$ with the transition probability
	\begin{equation} \label{eq:transNL} 
		P(x,A) = \int_{A}f(z-\mathcal{M}(x))dz,\hspace{3mm}x\in\mathbb{R}^{r}\text{ and }A\in\mathcal{B}^{r},
	\end{equation}
	where $ f(\cdot) $ is the density of $ e\ts{t} $.
	Observe that, from Assumption \ref{ass:f}, the transition density kernel in \eqref{eq:transNL} is positive everywhere, and thus $\{Y\ts{t}\}$ is $\nu_{r}$-irreducible. 
	
	We prove by showing that Tweedie's drift criterion \cite{tweedie1983} holds, i.e. there exists a small set $ G $ with $ \nu_r(G) > 0 $ and a non-negative continuous function $ \psi(x) $ such that 
	\begin{equation} \label{Tweedie1}
		E\{\psi(Y\ts{t})|Y\ts{t-1}=x\} \leq (1-\epsilon) \psi(x),\hspace{3mm}x \notin G,
	\end{equation}
	and
	\begin{equation} \label{Tweedie2}
		E\{\psi(Y\ts{t})|Y\ts{t-1}=x\} \leq M,\hspace{3mm}x\in G,
	\end{equation}
	for some $ 0<\epsilon < 1 $ and $ 0 < M < \infty$.

	Given that $\norm{\mathcal{M}(x)} \leq a\norm{x} + b$, where $a < 1$, we have
	\begin{align*} 
		& E\left(\left.\|Y\ts{t}\|^{\kappa}\right|Y\ts{t-1} = x
		\right) \\
		\leq \ & \norm{\mathcal{M}(x)}^{\kappa} + E\|e\ts{t}\|^{\kappa} \\
		\leq \ &  |a|^{\kappa}\norm{x}^{\kappa}+|b|^{\kappa}+ E\|e\ts{t}\|^{\kappa}.
	\end{align*}
	
	Define test function $\psi(x) = 1+\norm{x}^{\kappa} > 0 $. Then,
	\begin{align*} \label{eq0-3}
		& E\left(\left.\psi(Y\ts{t})\right|Y\ts{t-1}=x\right) \\
		\leq \ &  1+|a|^{\kappa}\norm{x}^{\kappa}+|b|^{\kappa} + E\|e\ts{t}\|^{\kappa} \\
		\leq \ & \rho \psi(x) + 1 - \rho + |b|^{\kappa} + E\|e\ts{t}\|^{\kappa},
	\end{align*}
	where $ \rho = |a|^\kappa < 1 $.
	
	Denote
	$ \epsilon = 1 - \rho - \frac{\left( 1 - \rho + |b|^{\kappa} + E\left(\norm{e\ts{t}}\right)^{\kappa}\right)}{\psi(x)} $ 
	and $ G = \{x : \norm{x} \leq L\} $ 
	such that $ \psi(x) > 1 + \frac{|b|^{\kappa}+E\|e\ts{t}\|^{\kappa}}{1-\rho} $ for all $ \norm{x} > L $. We obtain that conditions (\ref{Tweedie1}) and (\ref{Tweedie2}) hold.
	
	Moreover, $ E\left(\left.\phi(Y\ts{t})\right|Y\ts{t-1}=x\right) $
	is continuous with respect to $ x $ for any bounded continuous function $\phi(\cdot)$, then $\{Y\ts{t}\}$ is a Feller chain. By Feigin \& Tweedie \yrcite{feigin1985random}, $G$ is a small set. By referring to Theorem 4(ii) in Tweedie \yrcite{tweedie1983} and Theorem 1 in Feigin \& Tweedie \yrcite{feigin1985random}, $\{Y\ts{t}\}$ is geometrically ergodic with a unique strictly stationary solution.
	
	Furthermore, for a univariate stationary process $ \{Y\ts{t}\} $, by Harris Theorem \cite{meyn2012markov}, its autocovariance function satisfies $ \gamma_k = \text{Cov}(Y\ts{t+k}, Y\ts{t}) \leq \gamma_0 a^k $ for $ k \in \mathbb{N} $ and some $ 0 < a < 1 $. Thus, the autocovariance function is summable and the process has short memory.
\end{proof}

\subsection{Proof of Theorem \ref{thm:general2}} \label{sec:appendix-thm2}

\begin{proof}
	Proof of a similar result might exist in the literature, but we are unaware of the specific paper(s). For the convenience of the readers, we outline the proof here. 
	
	Let the Markov chain $\{Y\ts{t}\}$ and its state space be defined as in (i). Under the linear setting, model \eqref{eq:general-MC} can be written as
	\begin{equation} \label{eq1}
		Y\ts{t} = WY\ts{t-1} + e\ts{t},
	\end{equation}
	where $Y\ts{t}\in\mathbb{R}^{r}$ and $W\in\mathbb{R}^{r\times r}$, and the transition probability can be written as
	\begin{equation} \label{eq0}
		P(x,A) = \int_{A}f(z-Wx)dx,\hspace{3mm}x\in\mathbb{R}^{r}\text{ and }A\in\mathcal{B}^{r}.
	\end{equation}
	Under Assumption \ref{ass:f}, $\{Y\ts{t}\}$ is $\nu_{r}$-irreducible.
	
	First, suppose $\rho(W)<1$. Then, there exists an integer $s$ such that $\norm{W^{s}} < 1$. In the following, we prove that $s$-step Markov chain $\{Y\ts{ts}\}$ satisfies Tweedie's drift criterion \citep{tweedie1983}, i.e., there exists a small set $G$ with $\nu_r(G)>0$ and a non-negative continuous function $\psi(x)$ such that
	\begin{equation} 
		E\{\psi(Y\ts{ts})|Y\ts{(t-1)s} = x\} \leq (1-\epsilon)\psi(x), \hspace{3mm}x \notin G,
	\end{equation}
	and
	\begin{equation} 
		E\{\psi(Y\ts{ts})|Y\ts{(t-1)s} = x \} \leq M,\hspace{3mm}x\in G,
	\end{equation}
	for some constant $0<\epsilon<1$ and $0<M<\infty$.
	
	We iterate \eqref{eq1} $s$ times and obtain
	\begin{equation*}
		Y\ts{ts} = W^{s}Y\ts{(t-1)s} + \left(e\ts{ts}+\sum_{j=1}^{s-1}W^{j}e\ts{ts-j}\right).
	\end{equation*}
	Let $g(x) = 1+\norm{x}^{\kappa}$, and it can be verified that
	\begin{equation*}
		\begin{split}
			& E\{\psi(Y\ts{ts})|Y\ts{(t-1)s} = x\} \\ 
			\leq \ & 1+\norm{W^{s}}^{\kappa}\norm{x}^{\kappa} + E\left(e\ts{ts}+\sum_{j=1}^{s-1}W^{j}e\ts{ts-j}\right)\\
			\leq \ & \psi(x)\norm{W^{s}}^{\kappa} + C,
		\end{split}
	\end{equation*}
	where $C = 1 + E(e\ts{ts}+\sum_{j=1}^{s-1}W^{j}e\ts{ts-j}) - \norm{W^{s}}^{\kappa} <\infty $. Note that $\norm{W^{s}}^{\kappa} < 1$. Then there exists $L>0$, such that 
	\begin{equation*}
		E\{\psi(Y\ts{ts})|Y\ts{(t-1)s} = x \} \leq (1-\epsilon)\psi(x),\hspace{3mm} \forall \norm{x} > L,
	\end{equation*}and
	\begin{equation*}
		E\{\psi(Y\ts{ts})|Y\ts{(t-1)s} = x\} \leq M < \infty,\hspace{3mm} \forall \norm{x} \leq L.
	\end{equation*}
	and $G = \{x:\norm{x} \leq L\}$ with $\nu_{r}(G) > 0$.
	
	Moreover, because for each bounded continuous function $\phi(\cdot)$, $E\{\phi(Y\ts{ts})|Y\ts{(t-1)s} = x\}$ is continuous with respect to $x$, $\{Y\ts{ts}\}$ is a Feller chain. And $\{Y\ts{ts}\}$ is $\nu_r$-irreducible. This implies that $G$ is a small set \citep{feigin1985random}. By referring to Theorem 4(ii) in Tweedie \yrcite{tweedie1983},
	we can show that $\{Y\ts{ts}\}$ is geometrically ergodic with a unique strictly stationary solution. 
	By Lemma 3.1 of Tj{\o}stheim \yrcite{tjostheim1990non}, $\{Y\ts{t}\}$ is geometrically ergodic. 
	
	Then, we prove the necessity. 
	Suppose that model \eqref{eq1} is geometrically ergodic, then there exists a strictly stationary solution $ \{Y\ts{t}\} $ to model \eqref{eq1} \citep{feigin1985random}. And then the Markov chain $Y\ts{t}$ have a stationary distribution $\pi(\cdot)$, from which we can generate $Y\ts{0}$, and iteratively obtain the sequence $\{Y\ts{t}\}$. It is nonanticipative and equation \eqref{eq1} holds. 
	
	From \eqref{eq0}, it holds that 
	\begin{equation*}
		P(Y\ts{t}\in A|Y\ts{t-1} = x) = P(x, A) > 0
	\end{equation*}
	as $\nu_{r}(A) > 0$. Let $H$ be any affine invariant subspace of $\mathbb{R}^{r}$ under model \eqref{eq1}, i.e. $\{Wx+e\ts{t}:x\in H\}\subseteq H)$ with probability one . If $\nu_{r}(\mathbb{R}^r-H) \neq 0$, then for any $x\in H$, $P(Wx+e\ts{t}\in H) < 1$. As a result, $\mathbb{R}^r$ is the unique affine invariant subspace, and hence model $\eqref{eq1}$ is irreducible. Thus, by Theorem 2.5 in Bougerol \& Picard \yrcite{bougerol1992strict}, we have that the the top Lyapounov exponent is strictly negative, and thus spectral radius $\rho(W) = \norm{W^s}^{1/s} < 1$. This completes the proof of (ii).
	
\end{proof}

\subsection{Proof of Corollary \ref{cor:RNN-W}} \label{sec:appendix-rnn}

\begin{proof}
	Need to show that there always exist real numbers $ a<1 $ and $ b $ such that $\norm{\mathcal{M}_{\textsc{Rnn}} \left(u, v\right)} \leq a\norm{(u',v')'} + b$.
	
	Since $ g(\cdot) $ and $ \sigma(\cdot) $ are bounded, there exist positive constants $ M_1 $ and $ M_2 $ such that $ \| g(W_{zh} \sigma(W_{hh} v + W_{hy} u + b_h) + b_z)\|_{l_1} \leq M_1 $, $ \|\sigma(W_{hh} v + W_{hy} u + b_h)\|_{l_1} \leq M_2 $ for any $ u \in \mathbb{R}^p, v \in \mathbb{R}^q $. 
	
	Let $ a=a_0 \in (0, 1) $ and $ b = M_1 + M_2 $, we have 
	$\norm{\mathcal{M}_{\textsc{Rnn}} \left(u, v\right)}_{l_1} - a_0 \norm{(u',v')'}_{l_1} 
	\leq M_1 + M_2 - a_0 \|u\|_{l_1} - a_0 \|v\|_{l_1}
	\leq b = M_1 + M_2 $.
	By Theorem \ref{thm:general}, model \eqref{eq:RNN-MC} with bounded and continuous output and activation function is geometrically ergodic and has short memory.
\end{proof}

\subsection{Apply Theorem \ref{thm:general} to LSTM networks with \boldmath$p=q=1$}

We use an LSTM process with $p=q=1$ as an example to illustrate the application of Theorem \ref{thm:general} to LSTM networks, and prepare readers for Corollary \ref{cor:LSTM-W}. 
Assume that the norm $ \|\cdot\| $ in Theorem \ref{thm:general} is the $l_1$ norm. 
Although sigmoid is used by default as the activation functions for the gates, we also consider $ \sigma(\cdot) $ as ReLU or tanh for theoretical interests here. For output function $ g(\cdot) $, we consider commonly used linear, sigmoid and softmax functions. We summarize our results in Table A\ref{tab:LSTM}.

\begin{table*}[t]
	\caption{Application of Theorem \ref{thm:general} to specific LSTMs.}
	\label{tab:LSTM}
	\begin{center}
		\begin{small}
			\begin{tabular}{cccc}
				\toprule
				& & \multicolumn{2}{c}{Activation function $ \sigma $} \\
				\cmidrule[0.4pt](lr{0.125em}){3-4}%
				& & ReLU or identity & sigmoid or tanh \\
				\hline
				\multirow{3}{*}{\makecell{Output\\function\\$ g $}} & identity 
				&\makecell{$|w_{oh}| + |w_{ih}| +|w_{zh}w_{oh}|\leq a$, \\ $|w_{oy}| + |w_{iy}| + |w_{zh}w_{oy}|\leq a$, \\$|w_{fh}v+w_{fy}u+b_{f}|\leq a$}  
				& No  \\
				\cmidrule[0.4pt](lr{0.125em}){2-4}%
				& sigmoid 
				& \makecell{$|w_{oh}| + |w_{ih}|\leq a$, \\ $|w_{oy}| + |w_{iy}|\leq a$, \\$|w_{fh}v+w_{fy}u+b_{f}|\leq a$} 
				& $|\sigma(w_{fh}+w_{fy}+b_{f})|\leq a$ \\
				\cmidrule[0.4pt](lr{0.125em}){2-4}%
				& softmax 
				& \makecell{$|w_{oh}| + |w_{ih}|\leq a$, \\ $|w_{oy}| + |w_{iy}|\leq a$, \\ $|w_{fh}v+w_{fy}u+b_{f}|\leq a$} 
				& $|\sigma(w_{fh}+w_{fy}+b_{f})|\leq a$  \\
				\bottomrule
			\end{tabular}
		\end{small}
	\end{center}
	\vskip -0.1in
\end{table*}

\subsection{Proof of Corollary \ref{cor:LSTM-W}} \label{sec:appendix-lstm}

\begin{proof}
	Let $ a=a_0 \in (0, 1) $ and $ b = M + 2q $, we have 
	\begin{align*}
		& \norm{\mathcal{M}_{\textsc{Lstm}} \left(u, v, w\right)}_{l_1} - a_0 \norm{(u', v', w')'}_{l_1} \\
		\leq & ~ \|g(W_{zh} x + b_z)\|_{l_1} + \|x\|_{l_1} \\
		& \qquad + \|\boldsymbol{1}_q + f(u, v) \odot w\|_{l_1} - a_0 \norm{(u', v', w')'}_{l_1} \\
		\leq & ~ M + q + q + \|f(u, v)\|_{l_\infty} \|w\|_{l_1} \\
		& \qquad - a_0 \|u\|_{l_1} - a_0 \|v\|_{l_1} - a_0 \|w\|_{l_1} \\
		\leq & ~ M + 2q  - a_0 \|u\|_{l_1} - a_0 \|v\|_{l_1} \\
		& \qquad + (\|f(u, v)\|_{l_\infty}-a_0)\|w\|_{l_1} \nonumber \\
		\leq & ~ b = M + 2q,
	\end{align*}
	where
	$ x = o(u, v) \odot \tanh (i(u, v) \odot \tanh (W_{ch} v + W_{cy} u + b_c) + f(u, v) \odot w) \in B_\infty^q $,
	$ v = h\ts{t} \in B_\infty^q $, and
	$ u = y\ts{t-1} \in B_\infty^p $.
	The second inequality holds due to the definition of $ M $ and $ x \in B_\infty^q $.
	The forth inequality holds due to 
	$ \|f(u, v)\|_{l_\infty} = \|\sigma(W_{fh} v + W_{fy} u + b_f)\|_{l_\infty} 
	\leq \sigma(\|W_{fh}\|_{l_\infty} + \|W_{fy}\|_{l_\infty} + \|b_f\|_{l_\infty})
	\leq a_0. $
	
	By Theorem \ref{thm:general}, model \eqref{eq:LSTM-MC} is geometrically ergodic and has short memory.
\end{proof}

\subsection{Proof of Theorem \ref{thm:MRNN}} \label{sec:mrnn-mem}

\begin{proof}
	Without loss of generality, assume that the linear activation and output functions are identity. 
	
	(1) The RNN process can be written as
	\begin{equation*}
		\begin{cases}
			y\ts{t} = W_{zh} h\ts{t} + \varepsilon\ts{t} \\
			h\ts{t} = W_{hh} h\ts{t-1} + W_{hx} x\ts{t} 
		\end{cases}.
	\end{equation*}
	Then, $  h\ts{t} = (I - W_{hh}\mathcal{B})^{-1} W_{hx} x\ts{t} $, and we have
	\begin{equation*}
		y\ts{t} = W_{zh} (I - W_{hh}\mathcal{B})^{-1} W_{hx} x\ts{t} + \varepsilon\ts{t} .
	\end{equation*}
	Let $ y\ts{t} = \sum_{k=0}^{\infty} A_k x\ts{t-k} + \varepsilon\ts{t}$. 
	Since $ (I - W_{hh}\mathcal{B})^{-1} = \sum_{k=0}^{\infty} W_{hh}^k \mathcal{B}^k $, we have $ A_k = W_{zh} W_{hh}^k W_{hx} $, and $ (A_k)_{ij} $ decays exponentially for all $ i, j $.
	
	(2) The MRNNF process can be written as
	\begin{equation*}
		\begin{cases}
			y\ts{t} = W_{zh} h\ts{t} + W_{zm} m\ts{t} + \varepsilon\ts{t} \\
			h\ts{t} = W_{hh} h\ts{t-1} + W_{hx} x\ts{t} \\
			m\ts{t} = W_{mm} m\ts{t-1} + W_{mf} ((I-\mathcal{B})^d -I) x\ts{t}
		\end{cases}.
	\end{equation*}
	Then,
	\begin{equation*}
		\begin{cases}
			h\ts{t} = (I- W_{hh}\mathcal{B})^{-1} W_{hx} x\ts{t} \\
			m\ts{t} = (I- W_{mm}\mathcal{B})^{-1} W_{mf} ((I-\mathcal{B})^d -I) x\ts{t}
		\end{cases}.
	\end{equation*}
	Let $ y\ts{t} = \sum_{k=0}^{\infty} A_k x\ts{t-k} + \varepsilon\ts{t}$, then
	$ A_k = C_k + D_k $, where
	\begin{equation*}
		\begin{cases}
			\sum_{k=0}^{\infty} C_k x\ts{t-k} = W_{zh} (I- W_{hh}\mathcal{B})^{-1} W_{hx} x\ts{t} \\
			\sum_{k=0}^{\infty} D_k x\ts{t-k} = \\
			\qquad W_{zm} (I- W_{mm}\mathcal{B})^{-1} W_{mf} ((I-\mathcal{B})^d -I) x\ts{t}
		\end{cases}.
	\end{equation*}
	From part (1) we know that the entries in $ C_k $ decay exponentially as well as the entries in the first part $ W_{zm} (I- W_{mm}\mathcal{B})^{-1} $ in $ D_k $.
	Since $ (I-\mathcal{B})^d -I = \sum_{k=1}^{\infty} W_k \mathcal{B}^k $ and $ W_k $'s are diagonal matrices with $ (W_k)_{ii} \sim k^{-d_i-1} $, the decay of $ D_k $ is dominated by $ (I-\mathcal{B})^d -I $, and entries of $ A_k $ decay at some polynomial rate $ k^{-d_j-1} $.
\end{proof}

\subsection{Memory Property of Constant-gates-LSTM and Constant-gates-MLSTM} \label{sec:lstm-mem}

\begin{thm} 
	
	In terms of Definition \ref{def:mem-net}, the constant-gates-MLSTM has the capability of handling long-range dependence data, while the constant-gates-LSTM cannot.
\end{thm}

\begin{proof}
	Without loss of generality, assume that the linear activation and output functions are identity. 
	
	(1) The constant-gates-LSTM process can be written as
	\begin{equation*}
		\begin{cases}
			y\ts{t} = W_{zh} h\ts{t} + \varepsilon\ts{t} \\
			\tilde{c}\ts{t} = W_{ch} h\ts{t-1} + W_{cx} x\ts{t} \\
			c\ts{t} = D_i \, \tilde{c}\ts{t} + D_f \, c\ts{t-1} \\
			h\ts{t} = D_o \, c\ts{t}	
		\end{cases},
	\end{equation*}
	where $D_i$, $D_f$ and $D_o$ are matrices obtained by diagonalize the constant gates.
	
	Then, $ (I - D_f \mathcal{B}) c\ts{t} 
	= D_i \tilde{c}\ts{t} 
	= D_i (W_{ch} h\ts{t-1} + W_{cx} x\ts{t})
	= D_i W_{ch} D_o \, c\ts{t-1} + D_i W_{cx} x\ts{t} $,
	and we have $ (I - (D_f + D_i W_{ch} D_o) \mathcal{B}) c\ts{t} 
	= D_i W_{cx} x\ts{t} $. 
	Thus, $ c\ts{t} = (I - (D_f + D_i W_{ch} D_o) \mathcal{B})^{-1} D_i W_{cx} x\ts{t-1} $, then $ y\ts{t} = W_{zh} D_o (I - (D_f + D_i W_{ch} D_o) \mathcal{B})^{-1} D_i W_{cx} x\ts{t} + \varepsilon\ts{t} $. 
	From the proof of Theorem \ref{thm:MRNN} (1) we know that writing $ y\ts{t} = \sum_{k=0}^{\infty} A_k x\ts{t-k} + \varepsilon\ts{t}$, we have all the entries of $ A_k $ decay exponentially.
	
	(2) The constant-gates-MLSTM process can be written as
	\begin{equation*}
		\begin{cases}
			y\ts{t} = W_{zh} h\ts{t} + \varepsilon\ts{t} \\
			\tilde{c}\ts{t} = W_{ch} h\ts{t-1} + W_{cx} x\ts{t} \\
			(I-\mathcal{B})^{d} \, c\ts{t} = D_i \tilde{c}\ts{t} \\
			h\ts{t} = D_o c\ts{t}
		\end{cases}.
	\end{equation*}
	Then, $ (I - \mathcal{B})^{d} \, c\ts{t} 
	= D_i (W_{ch} h\ts{t-1} + W_{cx} x\ts{t}) 
	= D_i W_{ch} D_o c\ts{t-1} + D_i W_{cx} x\ts{t} $,
	and we have $ ((I - \mathcal{B})^{d} - D_i W_{ch} D_o \mathcal{B}) \, c\ts{t} 
	= D_i W_{cx} x\ts{t} $. 
	Thus, $ c\ts{t} = ((I - \mathcal{B})^{d} - D_i W_{ch} D_o \mathcal{B})^{-1} D_i W_{cx} x\ts{t} $, then $ y\ts{t} = W_{zh} D_o ((I - \mathcal{B})^{d} - D_i W_{ch} D_o \mathcal{B})^{-1} D_i W_{cx} x\ts{t} + \varepsilon\ts{t} $. 
	
	Now we need to obtain the rate of polynomial $ ((I - \mathcal{B})^{d} - C \mathcal{B})^{-1} $ for some matrix $ C = D_i W_{ch} D_o $.
	Let $ ((I - \mathcal{B})^{d} - C \mathcal{B})^{-1} = \sum_{j=0}^{\infty} \Theta_j \mathcal{B}^j $, then $ (\sum_{j=0}^{\infty} \Theta_j \mathcal{B}^j)((I - \mathcal{B})^{d} -C \mathcal{B}) = I $. Thus,
	\begin{align*}
		(\sum_{j=0}^{\infty} \Theta_j \mathcal{B}^j)(I - \mathcal{B})^{d} & = I + C \sum_{j=0}^{\infty} \Theta_j \mathcal{B}^{j+1} \\
		(\sum_{j=0}^{\infty} \Theta_j \mathcal{B}^j)(\sum_{k=0}^{\infty} W_k \mathcal{B}^k) & = I + C \sum_{j=0}^{\infty} \Theta_j \mathcal{B}^{j+1} \\
		\sum_{j=0}^{\infty} \sum_{k=0}^{\infty} \Theta_j \mathcal{B}^j  W_k \mathcal{B}^k  & = I + C \sum_{j=0}^{\infty} \Theta_j \mathcal{B}^{j+1}.
	\end{align*}
	Equate the coefficients for each $ B^j $ term for $ j = 0, 1, 2... $, we have 
	\begin{align*}
		\begin{cases}
			\Theta_0 & = I \\
			\Theta_1 & = C-W_1 \\
			\Theta_2 & = C\Theta_1 - W_1 \Theta_1 - W_2 \\
			\Theta_3 & = C\Theta_2 - W_1 \Theta_2 - W_2 \Theta_1 - W_3 \\
			& \vdots \\
			\Theta_k & = C \Theta_{k-1} - \sum_{j=1}^{k} W_j \Theta_{k-j}
		\end{cases}.
	\end{align*}
	The $ \Theta_k $'s are dominated by the $ W_k $ term and the elements decay at the same rate as $ W_k $, which is $ k^{-d_j-1} $. 
\end{proof}

\section{More Numerical Results}

\subsection{Autocorrelation Plots of All Datasets} \label{sec:acfs}

Autocorrelation plots of all 4 datasets, ARFIMA, DJI, traffic and tree, are shown in Figure \ref{fig:allacfs}.

\begin{figure*}[t]
	\centering
	\subfigure[ARFIMA]{\includegraphics[width=0.95\columnwidth]{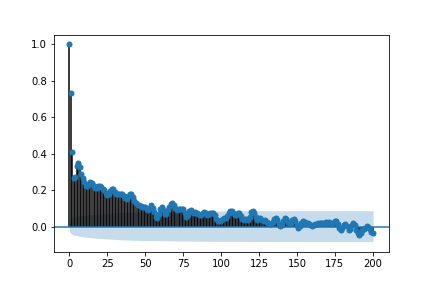}}
	\subfigure[DJI]{\includegraphics[width=0.95\columnwidth]{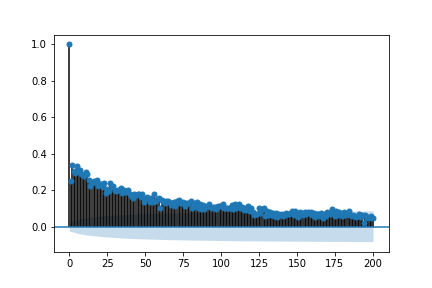}} \\
	\subfigure[traffic]{\includegraphics[width=0.95\columnwidth]{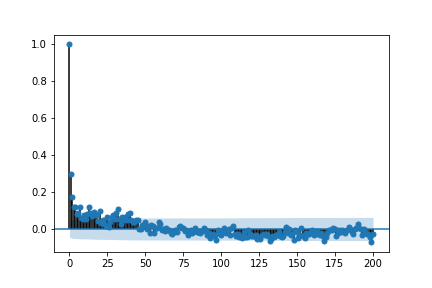}}
	\subfigure[tree]{\includegraphics[width=0.95\columnwidth]{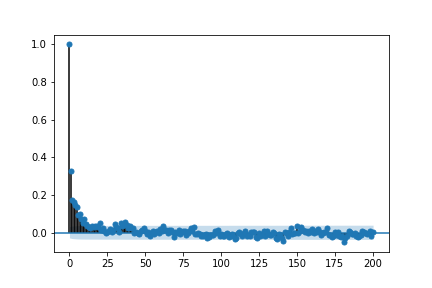}}
	\caption{Autocorrelation plots of all 4 datasets.}
	\label{fig:allacfs}
\end{figure*}

\subsection{Overall Performance of the Models} \label{sec:overall}

Average RMSE and standard deviation of one-step forecasting is reported in the main paper. We provide results in terms of MAE and MAPE, as well as figures, in this section.

\paragraph{RMSE}
Boxplot of RMSE for 100 different initializations are shown in Figure \ref{fig:overall-RMSE} for datasets ARFIMA, DJI, traffic and tree. 

\begin{figure*}[t]
	\centering
	\subfigure[ARFIMA]{\includegraphics[width=0.95\columnwidth]{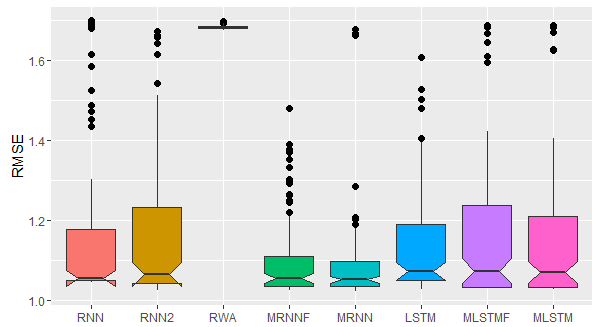}}
	\subfigure[DJI]{\includegraphics[width=0.95\columnwidth]{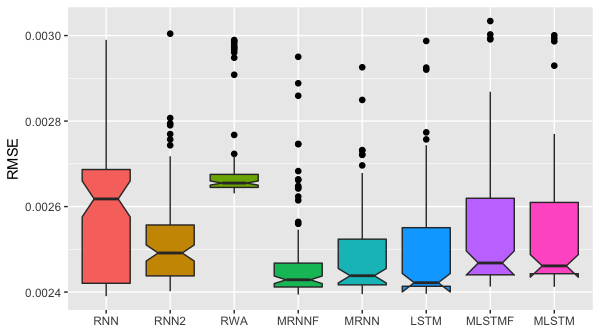}} \\
	\subfigure[traffic]{\includegraphics[width=0.95\columnwidth]{traffic-RMSE-box.png}}
	\subfigure[tree]{\includegraphics[width=0.95\columnwidth]{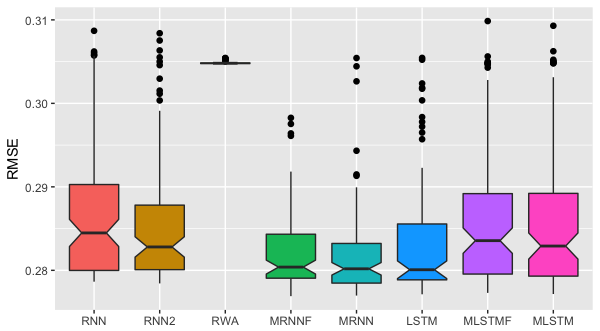}}
	\caption{Boxplot of RMSE for 100 different initializations.}
	\label{fig:overall-RMSE}
\end{figure*}

\paragraph{MAE}
Average MAE and standard deviation of one-step forecasting is shown in Table \ref{tab:average-mae}.

\begin{table}[t]
	\caption{Overall performance in terms of MAE. Average MAE and the standard deviation (in brackets) are reported.}
	\label{tab:average-mae}
	\vskip 0.15in
	\begin{center}
		\begin{small}
			\begin{tabular}{lcccc}
				\toprule
				& ARFIMA & DJI (x100)& Traffic & Tree \\
				\cmidrule[0.5pt]{1-5}
				RNN 	& \makecell{0.9310\\ (0.1550)} & \makecell{0.1977\\ (0.0242)} & \makecell{233.442\\ (12.391)} & \makecell{0.2240\\ (0.0064)} \\
				RNN2 	& \makecell{0.9310\\ (0.1430)} & \makecell{0.1861\\ (0.0164)} & \makecell{233.419\\ (12.378)} & \makecell{0.2229\\ (0.0057)} \\
				RWA 	& \makecell{1.3330\\ (0.0030)} 
				& \makecell{0.2052\\ (0.0164)} & \makecell{233.137\\ (7.425)} & \makecell{0.2379\\ (0.0001)} \\
				MRNNF 	& \makecell{0.8800\\ (0.0790)} & \makecell{\textbf{0.1809}\\ (0.0168)} & \makecell{\textbf{232.554}\\ (11.954)} & \makecell{0.2206\\ (0.0034)} \\
				MRNN 	& \makecell{\textbf{0.8710}\\ (0.0900)} 
				& \makecell{0.1835\\ (0.0165)} & \makecell{232.794\\ (12.149)} & \makecell{\textbf{0.2202}\\ (0.0037)} \\
				\cmidrule[0.5pt]{1-5}
				LSTM 	& \makecell{0.9070\\ (0.0940)} & \makecell{0.1841\\ (0.0182)} & \makecell{234.055\\ (11.149)} & \makecell{0.2215\\ (0.0051)} \\
				MLSTMF 	& \makecell{0.9240\\ (0.1320)} & \makecell{0.1895\\ (0.0203)} & \makecell{233.142\\ (11.551)} & \makecell{0.2235\\ (0.0060)} \\
				MLSTM 	& \makecell{0.9170\\ (0.1320)} & \makecell{0.1881\\ (0.0187)} & \makecell{233.035\\ (10.793)} & \makecell{0.2234\\ (0.0061)} \\
				\bottomrule 
			\end{tabular}
		\end{small}
	\end{center}
	\vskip -0.15in
\end{table}

Boxplot of MAE for 100 different initializations are shown in Figure \ref{fig:overall-MAE} for datasets ARFIMA, DJI, traffic and tree. 

\begin{figure*}[t]
	\centering
	\subfigure[ARFIMA]{\includegraphics[width=0.95\columnwidth]{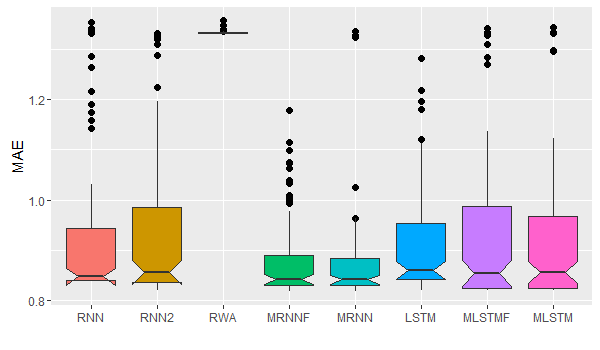}}
	\subfigure[DJI]{\includegraphics[width=0.95\columnwidth]{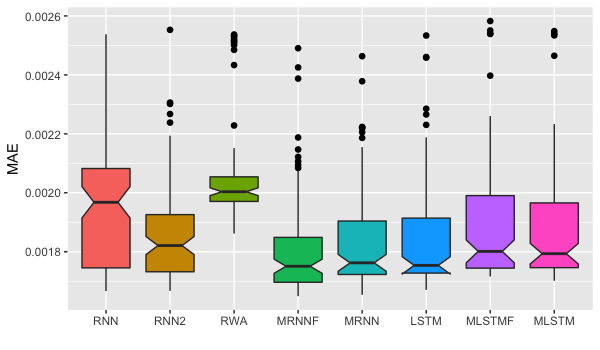}} \\
	\subfigure[traffic]{\includegraphics[width=0.95\columnwidth]{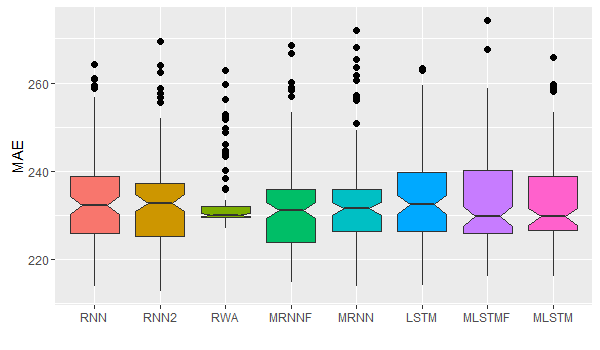}}
	\subfigure[tree]{\includegraphics[width=0.95\columnwidth]{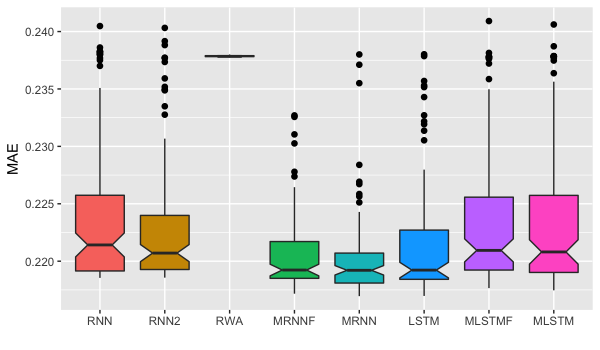}}
	\caption{Boxplot of MAE for 100 different initializations.}
	\label{fig:overall-MAE}
\end{figure*}

\paragraph{MAPE}
Average MAPE and standard deviation of one-step forecasting is shown in Table \ref{tab:average-mape}.

\begin{table}[t]
	\caption{Overall performance in terms of MAPE. Average MAPE and the standard deviation (in brackets) are reported.}
	\label{tab:average-mape}
	\vskip 0.15in
	\begin{center}
		\begin{small}
			\begin{tabular}{lcccc}
				\toprule
				& ARFIMA & DJI (x100)& Traffic & Tree \\
				\cmidrule[0.5pt]{1-5}
				RNN 	& \makecell{2.5760\\ (0.4030)} & \makecell{1.4371\\ (0.2566)} & \makecell{1.3943\\ (0.1998)} & \makecell{0.2747\\ (0.0079)} \\
				RNN2 	& \makecell{2.5570\\ (0.4420)} & \makecell{1.4407\\ (0.2106)} & \makecell{1.4092\\ (0.1789)} & \makecell{0.2739\\ (0.0071)} \\
				RWA 	& \makecell{\textbf{2.2370}\\ (0.1950)} 
				& \makecell{\textbf{1.2733}\\ (0.1702)} & \makecell{1.3745\\ ({0.1457})} & \makecell{0.2939\\ (0.0005)} \\
				MRNNF 	& \makecell{2.6430\\ (0.3380)} & \makecell{1.5561\\ (0.2243)} & \makecell{1.4270\\ (0.1834)} & \makecell{0.2714\\ (0.0042)} \\
				MRNN 	& \makecell{2.7010\\ (0.2680)} & \makecell{1.5031\\ (0.2045)} & \makecell{1.4253\\ (0.1586)} & \makecell{\textbf{0.2706}\\ (0.0044)} \\
				\cmidrule[0.5pt]{1-5}
				LSTM 	& \makecell{2.5660\\ (0.3750)} & \makecell{1.5725\\ (0.2283)} & \makecell{1.3632\\ (0.1807)} & \makecell{0.2727\\ (0.0060)} \\
				MLSTMF 	& \makecell{2.5100\\ (0.4690)} & \makecell{1.3141\\ (0.1369)} & \makecell{\textbf{1.3462}\\ (0.1769)} & \makecell{0.2750\\ (0.0074)} \\
				MLSTM 	& \makecell{2.5500\\ (0.4370)} & \makecell{1.3123\\ (0.1281)} & \makecell{1.3353\\ (0.1926)} & \makecell{0.2748\\ (0.0075)} \\
				\bottomrule 
			\end{tabular}
		\end{small}
	\end{center}
	\vskip -0.15in
\end{table}

Boxplot of MAPE for 100 different initializations are shown in Figure \ref{fig:overall-MAPE} for datasets ARFIMA, DJI, traffic and tree. 

\begin{figure*}[t]
	\centering
	\subfigure[ARFIMA]{\includegraphics[width=0.95\columnwidth]{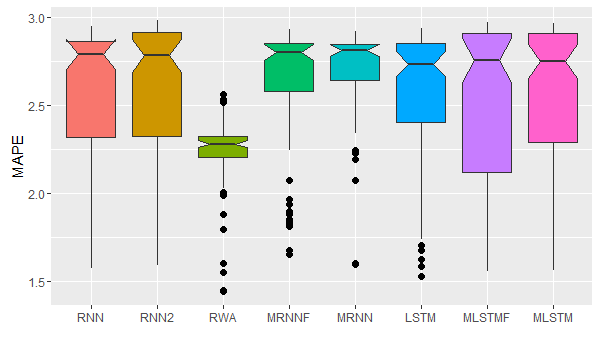}}
	\subfigure[DJI]{\includegraphics[width=0.95\columnwidth]{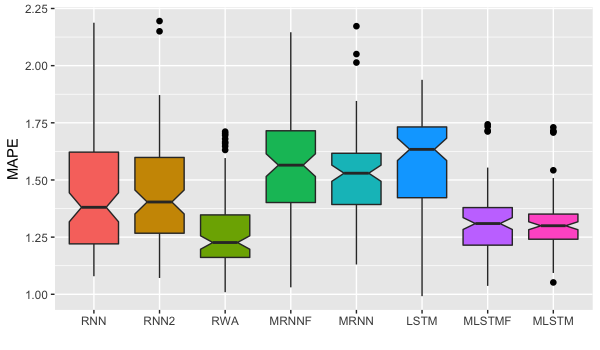}} \\
	\subfigure[traffic]{\includegraphics[width=0.95\columnwidth]{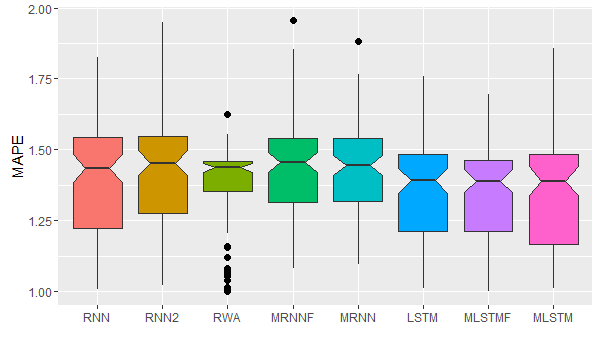}}
	\subfigure[tree]{\includegraphics[width=0.95\columnwidth]{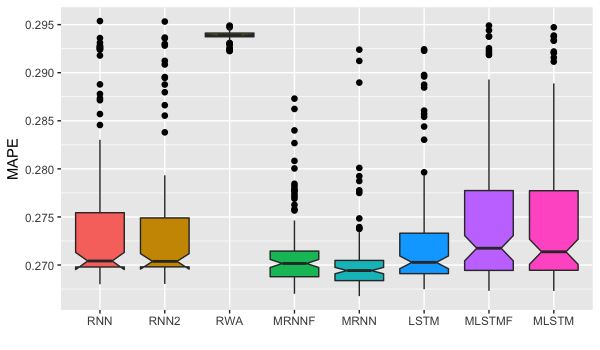}}
	\caption{Boxplot of MAPE for 100 different initializations.}
	\label{fig:overall-MAPE}
\end{figure*}

\subsection{Best Performance of the Models} \label{sec:best}

Best performance of the models, in terms of MAE and MAPE, are shown in Table \ref{tab:best-mae} \& \ref{tab:best-mape}.

\begin{table}[t]
	\caption{Best performance in terms of MAE.}
	\label{tab:best-mae}
	\vskip 0.15in
	\centering
	\begin{small}
		\begin{tabular}{lcccc}
			\toprule
			& ARFIMA & DJI (x100)& Traffic & Tree \\
			\cmidrule[0.5pt]{1-5} 
			ARFIMA 	& 0.8190 & 0.1800 	 & 230.99  & 0.2174 \\
			\cmidrule[0.5pt]{1-5}
			RNN 	& 0.8378 & 0.1667	 & 213.96  & 0.2185 \\
			RNN2 	& 0.8196 & 0.1667	 & \textbf{212.69} & 0.2186 \\
			RWA 	& 1.3307 & 0.1862	 & 227.01  & 0.2378 \\
			MRNNF 	& 0.8179 & \textbf{0.1649} 	 & 214.88  & 0.2172 \\
			MRNN 	& \textbf{0.8171} & 0.1654 & 213.79  & \textbf{0.2170} \\
			\cmidrule[0.5pt]{1-5} 
			LSTM 	& 0.8197 & 0.1671	 & 214.22  & {0.2170} \\
			MLSTMF	& 0.8191 & 0.1716	 & 216.15  & 0.2177 \\
			MLSTM 	& 0.8193 & 0.1702	 & 216.12  & 0.2175 \\
			\bottomrule 
		\end{tabular}
	\end{small}
\end{table}

\begin{table}[t]
	\caption{Best performance in terms of MAPE.}
	\label{tab:best-mape}
	\vskip 0.15in
	\centering
	\begin{small}
		\begin{tabular}{lcccc}
			\toprule
			& ARFIMA & DJI 		 & Traffic & Tree \\
			\cmidrule[0.5pt]{1-5} 
			ARFIMA 	& 2.8424 & 1.8334 	 & 1.6942 & 0.2676 \\
			\cmidrule[0.5pt]{1-5}
			RNN 	& 1.5729 & 1.0789	 & 1.0075 & 0.2680 \\
			RNN2 	& 1.5905 & 1.0714	 & 1.0215 & 0.2680 \\
			RWA 	& \textbf{1.4408} 
			& 1.0091 & \textbf{0.9986} & 0.2923 \\
			MRNNF 	& 1.6508 & 1.0304 	 & 1.0816 & 0.2670 \\
			MRNN 	& 1.5967 & 1.1303	 & 1.0938 & \textbf{0.2668} \\
			\cmidrule[0.5pt]{1-5} 
			LSTM 	& 1.5282 & \textbf{0.9918} & 1.0099 & 0.2675 \\
			MLSTMF	& 1.5565 & 1.0368	 & 0.9990 & 0.2673 \\
			MLSTM 	& 1.5597 & 1.0518	 & 1.0098 & 0.2673 \\
			\bottomrule 
		\end{tabular}
	\end{small}
\end{table}

\subsection{Performance on a Dataset without Long Memory} \label{sec:exp-RNN}

We generated a sequence of length 4001 ($2000+1200+800$) using model \eqref{eq:RNN-MC}, which does not have long memory according to Corollary \ref{cor:RNN-W}. We refer to this synthetic dataset as the RNN dataset. The boxplots of error measures are presented in Figure \ref{fig:overall-RNN}. From the boxplots we can see that the performance of our proposed models is comparable with that of the true model RNN, except that the variation of the error measures is a bit larger.

\begin{figure*}[t]
	\centering
	\subfigure[RMSE]{\includegraphics[width=0.95\columnwidth]{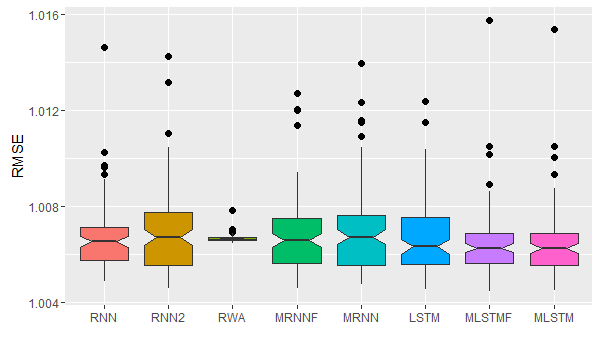}} \\
	\subfigure[MAE]{\includegraphics[width=0.95\columnwidth]{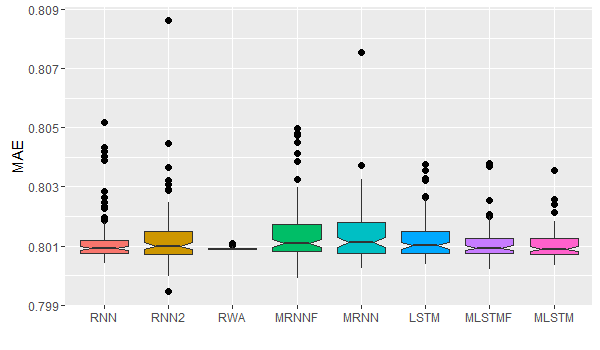}}
	\subfigure[MAPE]{\includegraphics[width=0.95\columnwidth]{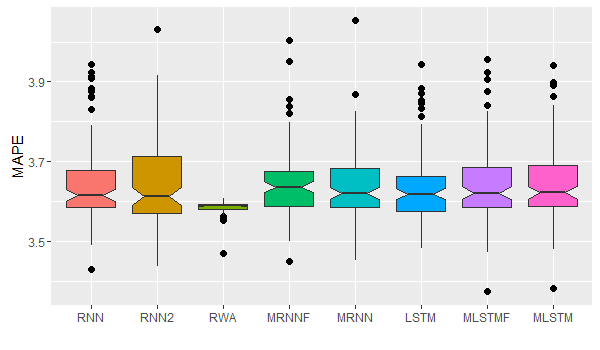}}
	\caption{Boxplot of RMSE, MAE and MAPE for 100 different initializations. Dataset: RNN.}
	\label{fig:overall-RNN}
\end{figure*}

\subsection{Experiment on Parameter $K$} \label{sec:exp-K}
Boxplot of RMSE for 100 different initializations are shown in Figure \ref{A-fig:arfima-RMSE-boxK}, \ref{A-fig:dji-RMSE-boxK}, \ref{A-fig:traffic-RMSE-boxK} and \ref{A-fig:tree-RMSE-boxK} for datasets ARFIMA, DJI, traffic and tree, respectively. 
Values of $K$ are appended to the abbreviations of the proposed models to distinguish the settings. 
For example, model ``MRNN25" means the MRNN model with $K=25$. 
There are 20 models with different settings in total, and they are sorted by the average RMSE in ascending order from left to right.

For MRNN and MRNNF, the prediction is generally better for a larger $K$, and they have smaller average RMSE than all the baseline models regardless of the choice of $K$. Interestingly, the performance of MLSTM and MLSTMF gets better when $K$ becomes smaller, and with $K=25$, they can outperform LSTM on ARFIMA and traffic datasets. Thus, we recommend a large $K$ for MRNN and MRNNF models, while for the more complicated MLSTM models, $K$ deserves more investigation to balance expressiveness and optimization.

\begin{figure*}[ht]
	\centering
	\includegraphics[width=0.95\textwidth]{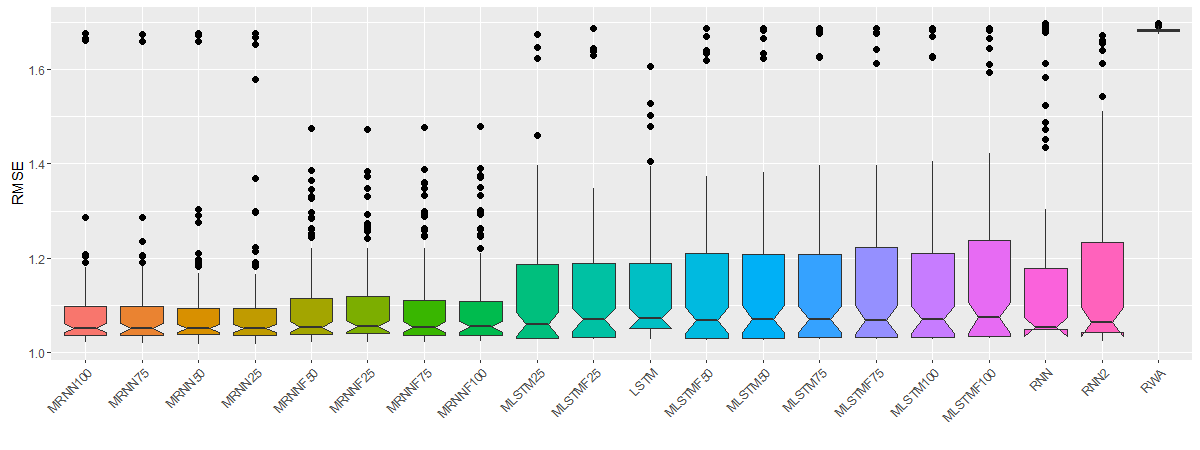}
	\caption{Boxplot of RMSE for 100 different initializations. Dataset: ARFIMA.}
	\label{A-fig:arfima-RMSE-boxK}
\end{figure*}

\begin{figure*}[ht]
	\centering
	\includegraphics[width=0.95\textwidth]{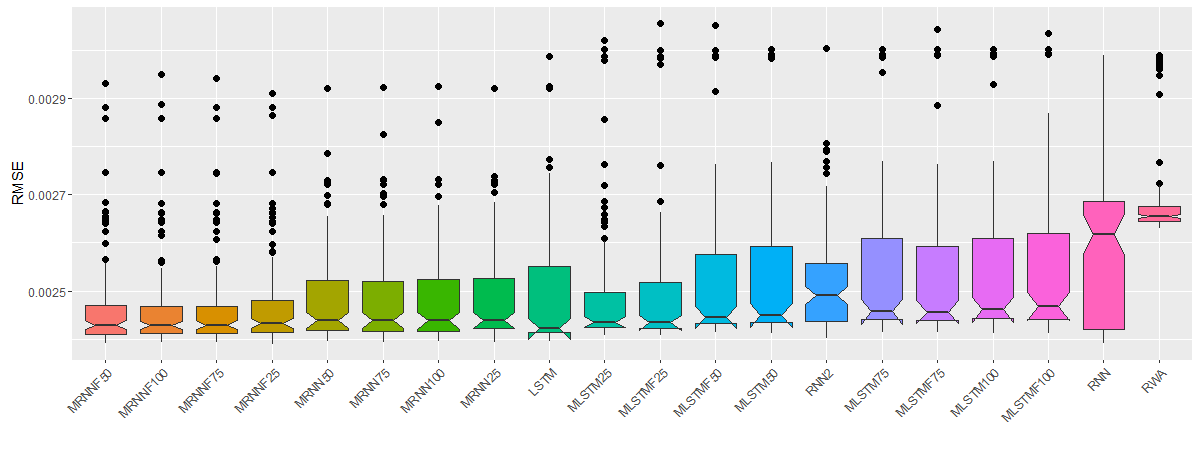}
	\caption{Boxplot of RMSE for 100 different initializations. Dataset: DJI.}
	\label{A-fig:dji-RMSE-boxK}
\end{figure*}

\begin{figure*}[ht]
	\centering
	\includegraphics[width=0.95\textwidth]{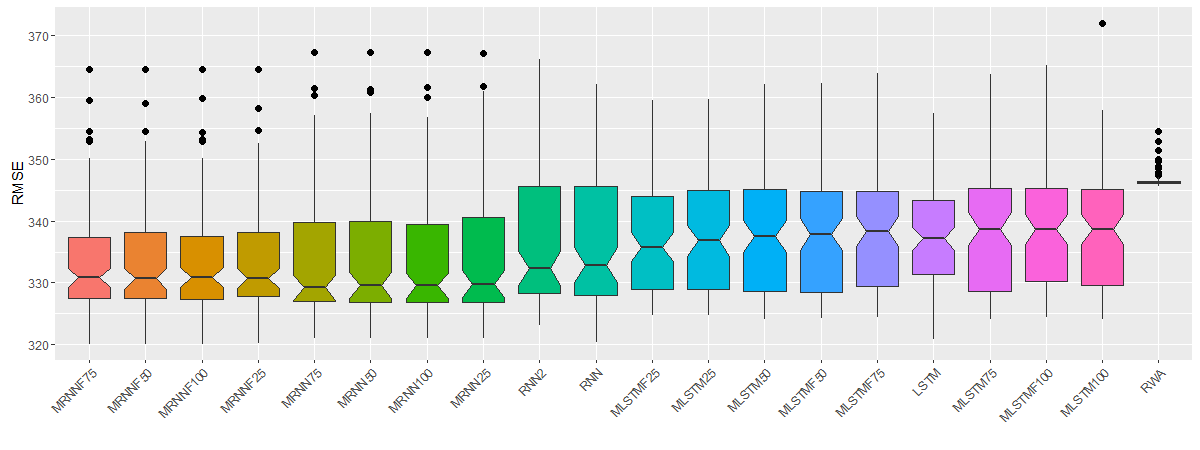}
	\caption{Boxplot of RMSE for 100 different initializations. Dataset: traffic.}
	\label{A-fig:traffic-RMSE-boxK}
\end{figure*}

\begin{figure*}[ht]
	\centering
	\includegraphics[width=0.95\textwidth]{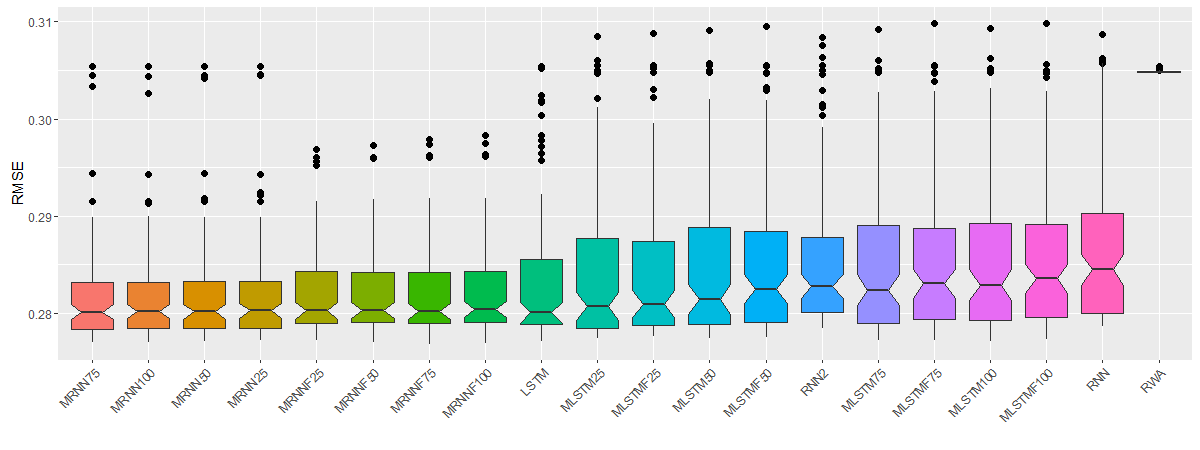}
	\caption{Boxplot of RMSE for 100 different initializations. Dataset: tree.}
	\label{A-fig:tree-RMSE-boxK}
	\vskip -0.15in
\end{figure*}

\end{document}